\newif\ifJOURNAL
\newif\ifCONF
\newif\ifarXiv
\newif\ifWP
\newif\ifFULL

\arXivtrue

\ifarXiv
  \documentclass[10pt]{article}
  \usepackage{amsmath,amsthm,amsfonts,amssymb,latexsym,graphicx}
\fi

\usepackage{etoolbox,stmaryrd,cite,enumitem} 
\usepackage{booktabs}  
\usepackage[shortcuts]{extdash} 

\allowdisplaybreaks[4]

\newtoggle{JOURNAL}
\newtoggle{CONF}
\newtoggle{arXiv}
\newtoggle{WP}
\newtoggle{FULL}
\newtoggle{TR}
\newtoggle{false} 
\newtoggle{true}\toggletrue{true} 

\ifJOURNAL\toggletrue{JOURNAL}\fi
\ifCONF\toggletrue{CONF}\fi
\ifarXiv\toggletrue{arXiv}\toggletrue{TR}\fi
\ifWP\toggletrue{WP}\toggletrue{TR}\fi
\ifFULL\toggletrue{FULL}\fi

\usepackage{environ}
\NewEnviron{wideformula}{%
  \iftoggle{CONF}{\begin{equation}\BODY\end{equation}}%
     {\begin{multline}\BODY\end{multline}}}
\NewEnviron{wideformula*}{%
  \iftoggle{CONF}{\begin{equation*}\BODY\end{equation*}}%
    {\begin{multline*}\BODY\end{multline*}}}
\newcommand{\NarrowLineBreak}{\iftoggle{CONF}{}{\\}}
\NewEnviron{widish}{\iftoggle{CONF}{$\BODY$ }{\[\BODY\]}}
\NewEnviron{widish.}{\iftoggle{CONF}{$\BODY$. }{\[\BODY.\]}}

\NewDocumentEnvironment{myproof}{o}
  {%
    \IfNoValueTF{#1}
      {\begin{proof}}
      {%
        \iftoggle{TR}
          {\begin{proof}[Proof {#1}]}
          {\medskip\begin{proof}\textbf{#1} }%
      }%
  }
  {\end{proof}}

\iftoggle{TR}{%
  \usepackage[authoryear]{natbib}
  \bibpunct{(}{)}{;}{a}{,}{,} 
  \usepackage[pdfpagemode=UseNone,pdfstartview=FitH]{hyperref}  
}{}%

\iftoggle{TR}{%
  \newtheorem{theorem}{Theorem}
  \newtheorem{proposition}[theorem]{Proposition}
  
  \newtheorem{corollary}[theorem]{Corollary}

  \theoremstyle{definition}
  
  \newtheorem{example}[theorem]{Example}

  \theoremstyle{remark}
  \newtheorem{remark}[theorem]{Remark}
}{}%

\newcommand{\N}{\mathbb{N}}   
\newcommand{\R}{\mathbb{R}}   

\newcommand{\e}{\mathrm{e}}    

\DeclareMathOperator{\URP}{\mathbb{P}^{\mathrm{R}}}  

\newcommand{\IRP}{\mathrm{IRP}}
\newcommand{\Sym}{\mathrm{Sym}}

\title{Inductive randomness predictors: beyond conformal}

\iftoggle{TR}{%
  \author{Vladimir Vovk}%
}{}%

\begin{document}
\maketitle
\begin{abstract}%
  This paper introduces inductive randomness predictors,
  which form a proper superset of inductive conformal predictors
  but have the same principal property of validity
  under the assumption of randomness (i.e., of IID data).
  It turns out that every non-trivial inductive conformal predictor
  is strictly dominated by an inductive randomness predictor,
  although the improvement is not great,
  at most a factor of $\e\approx2.72$ in the case of e-prediction.
  The dominating inductive randomness predictors are more complicated
  and more difficult to compute;
  besides, an improvement by a factor of $\e$ is rare.
  Therefore, this paper does not suggest replacing inductive conformal predictors
  by inductive randomness predictors
  and only calls for a more detailed study of the latter.
  \iftoggle{arXiv}{%

     The version of this paper at \url{http://alrw.net} (Working Paper 44)
     is updated most often.%
  }{}%
\end{abstract}

\section{Introduction}
\label{sec:introduction}

The assumption of randomness
(i.e., the observations being IID, independent and identically distributed)
is the fundamental one in mainstream machine learning.
Conformal predictors, in their basic form considered in this paper,
are guaranteed to satisfy a property of validity under randomness,
but they do not require randomness in order to be valid:
e.g., they remain valid if the assumption of randomness is weakened
to that of exchangeability.
A natural question is whether we can improve on conformal predictors
by using the assumption of randomness more fully.
The purpose of this paper is to show that some improvement is possible,
although it is unclear whether the improvement can be usefully exploited
in practice.

The first paper exploring this question was \citet{Nouretdinov/etal:2003ALT},
whose conclusion was that only a limited improvement is possible.
However, the setting of \citet{Nouretdinov/etal:2003ALT}
was the algorithmic theory of randomness,
and so their results involved unspecified constants.
A limited improvement might consist in improving conformal p-values
by a constant factor,
which ceases to be limited in practice if the factor is large enough.

Results developing those of \citet{Nouretdinov/etal:2003ALT}
and not involving unspecified constants were obtained in \citet{Vovk:arXiv2502}.
The latter paper introduced ``randomness predictors'',
the most general predictors enjoying the same property of validity
as conformal predictors under the assumption of randomness.
In fact, the definition of randomness predictors is trivial:
it is just a straightforward application of the definition of p-values.
Similarly to conformal e-predictors \citep{Vovk:2025PR},
the paper \citet{Vovk:arXiv2502} also introduced ``randomness e-predictors'',
based on e-values instead of p-values.
Since both conformal and randomness predictors are valid under the assumption of randomness,
the main advantage of randomness prediction, if real,
may lie in its \emph{efficiency}, which is defined, informally,
as the smallness of the p-values, or largeness of e-values,
that it produces for false labels.

Results of \citet{Vovk:arXiv2502} (see, e.g., Theorems 5 and 6) say, roughly,
that each randomness e-predictor can be turned
into a conformal e-predictor that loses in efficiency by at most a factor of $\e$
(the base of natural logarithms, $\e\approx2.72$):
the e-values for the false labels may go up at most $\e$-fold on average.
In practical machine learning and statistics,
an $\e$-fold improvement might be valuable.
A crude relation between e-values and p-values
is that an e-value of $e$ corresponds to a p-value of $1/e$
\citep[Remark 2.3]{Vovk/Wang:2021}.
This suggests that an improvement by a factor of $\e$
of conformal p-values might also be possible.
In this paper we will investigate whether and when
we can achieve it in reality.
The answer is that we can but, under our current definitions, not often.

The most popular kind of conformal predictors is inductive conformal predictors (ICPs).
They split the training set of size $n$ into two parts:
a proper training set of size $l$ and a calibration set of size $m$,
where $l+m=n$.
The main advantage of ICPs is that they can be used on top of generic point predictors
(such as neural networks)
without prohibitive computational costs,
whereas full conformal prediction is computationally efficient only
on top of a relatively narrow class of point predictors.
This paper introduces and studies inductive randomness predictors (IRPs),
which are also computationally efficient
(provided results of some preliminary computations,
which only depend on the size $m$ of the calibration set rather than the actual data,
are stored as a table).

A major limitation of conformal predictors,
discussed in detail in \citet{Vovk/etal:2009AOS},
is that the p-values that they output can never drop below $\frac{1}{n+1}$.
Correspondingly, the smallest p-value that can be achieved by an ICP
is $\frac{1}{m+1}$.
Let us call it the \emph{fundamental limitation of inductive conformal prediction}.
All specific IRPs considered in this paper break through this limitation:
they are capable of achieving p-values of $\frac{1}{\e(m+1)}$.
The factor of $\e$ is negligible by the standards of the algorithmic theory of randomness,
but substantial by the usual standards of machine learning and statistics.
(In principle,
it is easy to overcome the fundamental limitation of inductive conformal prediction
by using smoothed ICPs \citep[Sect.~4.2.1]{Vovk/etal:2022book},
but randomization is often considered problematic and avoided in practice.)

We will start the main part of the paper in Sect.~\ref{sec:definitions}
from the principal definitions, including that of IRPs.
Similarly to ICPs,
IRPs are defined using inductive nonconformity measures,
but now these take values in a ``summary space'' $\mathbf{S}\subseteq\R$.
The size of the summary space has important implications for the achievable randomness p-values
(i.e., p-values output by inductive randomness predictors).
In Sect.~\ref{sec:binary} we discuss binary IRPs,
corresponding to $\left|\mathbf{S}\right|=2$.
This case leads to the smallest randomness p-values,
but on the negative side binary inductive nonconformity measures may be crude,
which will be illustrated on two examples.

In Sect.~\ref{sec:inadmissible} we will see that the idea of binary IRPs
can be used to demonstrate the inadmissibility of ICPs,
where ``inadmissibility'' means, according to the standard usage in statistical decision theory,
that for each ICP there exists an IRP that is never worse and sometimes better
than that ICP.
The following section, Sect.~\ref{sec:SIRP},
shows that ICPs can be improved in a much stronger sense.
It allows the summary space $\mathbf{S}$
to be the real line or its infinite subset,
and it introduces a class of IRPs which we call ``separation IRPs''.
Informally, separation IRPs are obtained by combining inductive conformal prediction
with repeated application of binary inductive randomness prediction.
Each non-trivial ICP is strictly dominated by a corresponding IRP;
moreover,
in typical cases
a separation IRP based on an inductive nonconformity measure $A$ produces a p-value
that is almost surely better
than the p-value produced by the ICP based on $A$.
Besides, similarly to the binary IRPs,
the separation IRPs can still achieve randomness p-values of $\frac{1}{\e(m+1)}$
breaking the fundamental limitation.

However, separation IRPs still have a substantial limitation:
in general (without restrictions on the summary space $\mathbf{S}$),
they can achieve a p-value at best $K/(m+1)$
when the corresponding ICP achieves a p-value of $(K+1)/(m+1)$;
therefore, there can be a significant improvement
only for very small values of $K$, first of all for $K=0$.
This motivates the further study of discrete IRPs,
for which $\left|\mathbf{S}\right|<\infty$,
in the following two sections.
In the ternary case $\left|\mathbf{S}\right|=3$,
which is the topic of Sect.~\ref{sec:ternary},
we have finer inductive nonconformity measures than in the binary case,
and the ternary case might be a better contender to be useful in practice.
Section~\ref{sec:discrete} extends these considerations
to an arbitrary finite summary space $\mathbf{S}$.
The quaternary case $\left|\mathbf{S}\right|=4$
might be another practically useful one,
along the lines of \citet[Figure~1.5]{Vovk/etal:2022book}.

The short Sect.~\ref{sec:conclusion} concludes.
The proofs are relegated to Appendix~\ref{app:proofs}.

Let $\N_0:=\{0,1,\dots\}$ and $\N_1:=\{1,2,\dots\}$ be the two standard sets of natural numbers.

\section{Inductive conformal and randomness predictors}
\label{sec:definitions}

We consider the problem of batch prediction.
Given a training sequence $z_1,\dots,z_n$ of a given length $n$,
where $z_i=(x_i,y_i)$ (an \emph{example})
consists of an \emph{object} $x_i\in\mathbf{X}$
and a \emph{label} $y_i\in\mathbf{Y}$,
and also given a test object $x_{n+1}\in\mathbf{X}$,
our task is to predict the label $y_{n+1}$ of $x_{n+1}$.
The object space $\mathbf{X}$ and the label space $\mathbf{Y}$
are non-empty measurable spaces.
To exclude trivialities, let us assume that $n\ge2$
and that the $\sigma$-algebra on $\mathbf{Y}$
is different from $\{\emptyset,\mathbf{Y}\}$
(i.e., that $\mathbf{Y}$ contains at least two essentially distinct elements).

In the definition of an ICP
we will follow~\citet[Sect.~4.2.2]{Vovk/etal:2022book}.
The training sequence $z_1,\dots,z_n$ is split into two parts:
the \emph{proper training sequence} $z_1,\dots,z_l$ of length $l$
and the \emph{calibration sequence} $z_{l+1},\dots,z_n$ of length $m:=n-l$;
we will assume $l\in\N_1$ and $m\in\N_1$.
An \emph{inductive nonconformity measure} is a measurable function
$A:\mathbf{Z}^{l+1}\to\R$,
where $\mathbf{Z}:=\mathbf{X}\times\mathbf{Y}$ is the \emph{example space}.
The \emph{inductive conformal predictor} (ICP) based on $A$
outputs the prediction p-function
\begin{equation*} 
  f(y)
  :=
  \frac
    {\left|\left\{j=l+1,\dots,n+1\mid\alpha_j\ge\alpha_{n+1}\right\}\right|}
    {m+1}
  \in
  \left[
    \frac{1}{m+1},
    1
  \right],
  \quad
  y\in\mathbf{Y},
\end{equation*}
where the $\alpha$s are defined by
\begin{align}
  \alpha_j &:= A(z_1,\dots,z_l,z_j), \quad j=l+1,\dots,n,
  \iftoggle{FULL}{\label{eq:alpha-1}}{\notag}\\
  \alpha_{n+1} &:= A(z_1,\dots,z_l,(x_{n+1},y)).
  \iftoggle{FULL}{\label{eq:alpha-2}}{\notag}
\end{align}
We often refer to the values $\alpha$ taken by an inductive nonconformity measure
as \emph{nonconformity scores}.
There are different ways of packaging predictions output by ICPs
(such as prediction sets, briefly discussed after introducing IRPs below).

To define and discuss IRPs,
we will need several auxiliary notions.
The \emph{upper randomness probability} of a measurable set $E\subseteq\mathbf{Z}^{n+1}$
is defined in \citet[Sect.~9.1.1]{Vovk/etal:2022book} as
\begin{equation}\label{eq:URP}
  \URP(E)
  :=
  \sup_{Q\in\mathfrak{P}(\mathbf{Z})}
  Q^{n+1}(E),
\end{equation}
where we use the notation $\mathfrak{P}(Z)$
for the set of all probability measures on a measurable set $Z$.
A \emph{randomness p-variable} on $\mathbf{Z}^{n+1}$
is a measurable function $P:\mathbf{Z}^{n+1}\to[0,1]$
satisfying
\begin{equation}\label{eq:p-variable}
  \forall\epsilon\in(0,1):
  \URP(\{P\le\epsilon\})
  \le
  \epsilon.
\end{equation}
A \emph{randomness p-predictor}, as defined in \citet{Vovk:arXiv2502},
is the same thing as a randomness p-variable.
We will see that this terminology is justified
after the definition of the IRPs,
which are a subclass of randomness p-predictors,
below.

Very slightly generalizing the notion used when defining ICPs,
an \emph{inductive nonconformity measure} used for IRPs
is a measurable function
$A:\mathbf{Z}^{l+1}\to\mathbf{S}$,
where $\mathbf{S}$ is a measurable space
which we will call the \emph{summary space}.
We will assume that $\mathbf{S}\subseteq\R$
and that $\mathbf{S}$ inherits
the structures of measurable, topological, and linearly ordered space from $\R$
(so that it can be argued that the new definition is not a generalization at all).
Similarly to \eqref{eq:URP}, we define the upper randomness probability
of a measurable set $E\subseteq\mathbf{S}^{m+1}$ as
\begin{equation*}
  \URP(E)
  :=
  \sup_{Q\in\mathfrak{P}(\mathbf{S})}
  Q^{m+1}(E).
\end{equation*}
(Therefore, the notation $\URP$ is overloaded,
but it should never lead to confusion in this paper.)
An \emph{aggregating p-variable} $P:\mathbf{S}^{m+1}\to[0,1]$
is defined to be a randomness p-variable on $\mathbf{S}^{m+1}$,
meaning that it is required to satisfy \eqref{eq:p-variable}.

In inductive randomness prediction,
the training sequence $z_1,\dots,z_n$ is still split
into the proper training sequence $z_1,\dots,z_l$
and the calibration sequence $z_{l+1},\dots,z_n$.
The \emph{inductive randomness predictor} (IRP) based on
(sometimes we will say ``corresponding to'')
an inductive nonconformity measure $A$ and an aggregating p-variable $P$
is defined to be the randomness p-predictor
\begin{equation*} 
  \IRP_{A,P}(z_1,\dots,z_{n+1})
  :=
  P(\alpha_{l+1},\dots,\alpha_{n+1}),
\end{equation*}
where
\begin{equation}
  \iftoggle{FULL}{\label{eq:alpha}}{\notag}
  \alpha_j := A(z_1,\dots,z_l,z_j), \quad j=l+1,\dots,n+1.
\end{equation}
Given a training sequence $z_1,\dots,z_n$ and a test object $x_{n+1}$,
the IRP $\IRP_{A,P}$ outputs the prediction p-function
\begin{equation}\label{eq:f}
  f(y)
  =
  f(y;z_1,\dots,z_n,x_{n+1})
  :=
  \IRP_{A,P}(z_1,\dots,z_n,(x_{n+1},y)),
  \quad
  y\in\mathbf{Y}.
\end{equation}
This function itself can be considered to be the IRP's prediction
for $y_{n+1}$.
Alternatively, we can choose a \emph{significance level} $\epsilon>0$
(i.e., our target probability of error)
and output the prediction set
\begin{equation}\label{eq:Gamma}
  \Gamma^{\epsilon}
  :=
  \left\{
    y\in\mathbf{Y}
    \mid
    f(y)>\epsilon
  \right\}
\end{equation}
as our prediction for $y_{n+1}$.
By the definition of a randomness p-variable,
the probability of error (meaning $y_{n+1}\notin\Gamma^{\epsilon}$)
will not exceed $\epsilon$ under the assumption of randomness.

We will only be interested in IRPs
for which their underlying aggregating p-variable
is \emph{calibration-invariant},
i.e., does not depend on the ordering of its first $m$ arguments
(corresponding to the calibration examples),
so we make this requirement part of the definition.
This makes the IRPs themselves independent of the ordering of the calibration examples.

\begin{remark}
  \upshape
  In our analysis of IRPs,
  we will assume that all $n+1$ examples under consideration are IID,
  although it will be obvious that it is sufficient to assume
  that only the calibration and test examples are IID.
\end{remark}

ICPs are a special case of IRPs;
for them, $\mathbf{S}=\R$, and they are based on the aggregating p-variable
\begin{wideformula}
  \iftoggle{FULL}{\label{eq:Pi}}{\notag}
  \Pi(\alpha_{l+1},\dots,\alpha_{n+1})
  :=
  \frac{\left|\left\{j=l+1,\dots,n+1\mid\alpha_j\ge\alpha_{n+1}\right\}\right|}{m+1},
  \NarrowLineBreak
  \enspace (\alpha_{l+1},\dots,\alpha_{n+1})\in\mathbf{S}^{m+1}.
\end{wideformula}
Therefore, we will use the notation $\IRP_{A,\Pi}$ for the ICP
based on an inductive nonconformity measure $A$.

In statistical hypothesis testing
(see, e.g., \citealt[Sect.~3.2]{Cox/Hinkley:1974})
it is customary to define p-variables via ``test statistics'',
$B:\mathbf{S}^{m+1}\to\R$ in our current context.
This notion is insufficient in conformal prediction
\citep{Gurevich/Vovk:2019COPA},
where we need to replace $\R$ by a general linearly ordered measurable space
with all initial segments $(-\infty,r]$ measurable.
Let us call functions $B$ of this kind \emph{nonconformity statistics}.
\iftoggle{FULL}{\bluebegin
  Replacing ``test statistic'' by ``nonconformity statistic''
  serves a double role:
  emphasizes greater generality;
  makes it easy to remember that large values of the statistics are significant
  (otherwise we say ``conformity statistic'').
\blueend}{}%
Such a function defines the aggregating p-variable
\begin{wideformula}\label{eq:P_B}
  P_B(\alpha_{l+1},\dots,\alpha_{n+1})
  :=
  \URP
  \left(\left\{
    B
    \ge
    B(\alpha_{l+1},\dots,\alpha_n,\alpha_{n+1})
  \right\}\right),\NarrowLineBreak
  \enspace (\alpha_{l+1},\dots,\alpha_{n+1})\in\mathbf{S}^{m+1}.
\end{wideformula}
(Intuitively, large values of $B$ indicate nonconformity.
We replace the ``$\ge$'' in \eqref{eq:P_B} by ``$\le$''
when $B$ is referred to as a ``conformity statistic''.)
This aggregating p\-/variable can then be used as an input to an IRP,
and then we might say that this IRP,
$\IRP_{A,P_B}$,
is based on $A$ (an inductive nonconformity measure) and $B$.

\section{Binary inductive randomness predictors}
\label{sec:binary}

In this section we will concentrate on \emph{binary IRPs},
for which the summary space is $\mathbf{S}:=\{0,1\}$.
Intuitively, a summary of 0 means conformity,
and 1 means lack of conformity.
Binary IRPs are simple and even crude;
however, they are able to output smaller p-values
than other IRPs considered in this paper.

Binary IRPs will output prediction p-functions of an especially simple kind.
Namely, for them the prediction function \eqref{eq:f} will be a \emph{hedged prediction set}
in the sense of having the form
\begin{equation}\label{eq:hedged}
  f(y)
  =
  \begin{cases}
    c' & \text{if $y\in E$}\\
    c & \text{otherwise}
  \end{cases}
\end{equation}
for some $E\subseteq\mathbf{Y}$ and $c,c'\in[0,1]$ with $c\le c'$
(typically $c<c'$, so that $f$ identifies $E$ uniquely).
We call $E$ the \emph{prediction set} associated with $f$,
$c$ is the \emph{unconfidence} in $E$ (and $1-c$ is the \emph{confidence}),
and $c'$ is the \emph{credibility} of the prediction $f$.
(These terms are used similarly to conformal prediction,
as in \citealt[Sections 3.1.2 and 3.5.1]{Vovk/etal:2022book}.)
We will be mainly interested in $E$ and $c$;
$c$ reflects our confidence in the prediction set $E$;
the smaller $c$ the greater our confidence.
As always, the expression ``prediction interval'' will be applied
to prediction sets that happen to be intervals of the real line,
and the corresponding hedged prediction sets will be called hedged prediction intervals.

The nonconformity statistic $B$ used by binary IRPs is
\begin{equation}\label{eq:B-binary}
  B(\alpha_{l+1},\dots,\alpha_n,\alpha_{n+1})
  :=
  \left(
    \alpha_{n+1},
    -\sum_{i=l+1}^n\alpha_i
  \right);
\end{equation}
it takes values in $\R^2$ equipped with the lexicographic order.
Remember that $(\alpha,\beta)\le(\alpha',\beta')$ in the lexicographic order
means that either $\alpha<\alpha'$ or $\alpha=\alpha'$ and $\beta\le\beta'$.
This gives us a linear order (so that every two elements of $\R^2$ are comparable).

Our definition \eqref{eq:B-binary} is the most natural choice
for the nonconformity statistic:
the nonconformity of a sequence $(\alpha_{l+1},\dots,\alpha_{n+1})$
of nonconformity scores is determined
by the nonconformity score $\alpha_{n+1}$ of the test example
except that ties are broken by the total nonconformity score
for the calibration sequence.
We will get an equivalent definition
if we replace the $\sum$ in \eqref{eq:B-binary}
by any other symmetric function that is strictly increasing
in each of its arguments;
remember that we are only interested in calibration-invariant IRPs.
The most nonconforming $(\alpha_{l+1},\dots,\alpha_{n+1})$
correspond to nonconforming $\alpha_{n+1}$
and conforming $(\alpha_{l+1},\dots,\alpha_{n})$.

Once we fix the nonconformity statistic \eqref{eq:B-binary},
a binary IRP is determined by its inductive nonconformity measure.
Informally, there are two kinds of binary inductive nonconformity measures,
which we will call intrinsic and extrinsic.
Extrinsic ones are obtained from non-binary (usually continuous)
inductive nonconformity measures by thresholding;
the nonconformity scores above the threshold are replaced by 1,
while those below are replaced by 0.
Intrinsic ones are defined in some other natural way
without need for thresholding at the last step.
There might be intermediate cases,
e.g., where the binary inductive nonconformity measure
uses thresholding, but the threshold is defined in a natural way.
Let me give two examples of binary inductive nonconformity measures
illustrating the two kinds.
(Two more examples of inductive nonconformity measures will be given
at the end of this section and, as Example~\ref{ex:standard},
in Sect.~\ref{sec:SIRP}.)

\begin{example}\label{ex:decision}
  \upshape
  Train a decision tree $\hat g$ on the proper training set
  in such a way that,
  when applied to an example $z\in\mathbf{Z}$,
  $\hat g$ outputs either 0 or 1,
  where 1 is intended to be an indication of the strangeness of $z$
  as compared with the examples in the proper training set.
  The nonconformity score of a calibration or test example $z$ is then $\hat g(z)$
  (where the test example is the test object plus a postulated label $y$).
  This is an example of an intrinsic binary nonconformity measure.
  We can use \eqref{eq:B-binary} as nonconformity statistic.
  The hedged prediction set for a test object $x_{n+1}$
  will be the set of all labels $y$ such that
  the nonconformity score of $(x_{n+1},y)$ is 0;
  its unconfidence will be given by Proposition~\ref{prop:binary} below
  (and discussed after the proposition).
\end{example}

\begin{example}\label{ex:SVM}
  \upshape
  Now set $\mathbf{Y}:=\{-1,1\}$,
  so that here we are interested in binary classification.
  To define the nonconformity score $A(z_1,\dots,z_l,(x,y))$,
  use the support vector machine (SVM) to find
  the optimal separating hyperplane and its margin
  from $z_1,\dots,z_l$ as training sequence.
  Set the nonconformity score $A(z_1,\dots,z_l,(x,y))$ to 1
  if $x$ is classified incorrectly (namely, as $-y$)
  by the optimal separating hyperplane
  and $x$ is outside the margin;
  set $A(z_1,\dots,z_l,(x,y))$ to 0 otherwise.
  A reasonable definition of $B$ is still \eqref{eq:B-binary}.
  The prediction set produced by this IRP for a test object $x_{n+1}$
  will be $\{\hat y\}$ if $x_{n+1}$ is outside the margin,
  where $\hat y$ is the SVM's prediction for the label of $x_{n+1}$.
  Otherwise (if $x_{n+1}$ is inside the margin),
  the prediction set will be vacuous, $\{-1,1\}$.
  The unconfidence of this prediction set will again be given
  in Proposition~\ref{prop:binary} below.
  This definition may be regarded as extrinsic,
  although we may argue that the threshold is natural.

  An alternative definition would be to set $A(z_1,\dots,z_{l},(x,y))$ to 1
  if $x$ is a support vector for the SVM
  constructed from $(z_1,\dots,z_{l},(x,y))$ as training sequence
  and to set it to 0 otherwise,
  as in \citet[Sect.~2]{Gammerman/etal:1998}
  (that paper uses ``incertitude'' for our ``unconfidence'').
  However, the computational cost of such an IRP would be prohibitive,
  since it would require constructing a new SVM
  for each test object and each possible label for it.
  This alternative definition would be intrinsic.
\end{example}

Binary IRPs,
including the IRPs described in Examples~\ref{ex:decision} and~\ref{ex:SVM},
output prediction sets that do not depend on the calibration sequence.
This makes them inflexible as compared with typical conformal predictors,
but on the positive side they can achieve very low unconfidences.

Let us now derive an expression (not quite explicit)
for the p-values output
by binary IRPs based on the nonconformity statistic \eqref{eq:B-binary}.
The following proposition,
to be proved in Sect.~\ref{app:binary}, gives the expression,
and after its statement we will discuss ways of using it.

\begin{proposition}\label{prop:binary}
  Suppose that a binary sequence $\alpha_{l+1},\dots,\alpha_{n}$
  contains $K$ 1s
  and that $\alpha_{n+1}=1$.
  Then the nonconformity statistic $B$ defined by \eqref{eq:B-binary}
  leads to a p-value $P_B(\alpha_{l+1},\dots,\alpha_{n+1})$ of
  \begin{equation}\label{eq:main}
    \IRP^1(m,K)
    :=
    \max_{p\in[0,1]}
    \sum_{k=0}^K
    \binom{m}{k}
    p^{k+1}
    (1-p)^{m-k}.
  \end{equation}
  Let $m\to\infty$.
  \begin{itemize}
  \item
    For $K=0$, the p-value is
    \begin{equation}\label{eq:K-0}
      \IRP^1(m,0)
      =
      \frac{m^m}{(m+1)^{m+1}}
      \sim
      \frac{\exp(-1)}{m}
      \approx
      \frac{0.37}{m}
    \end{equation}
    (and we can replace ``$\sim$'' by ``$\le$'').
  \item
    For $K=1$, the p-value is asymptotically equivalent
    to $c/m$, where $c\in[0.83,0.84]$.
  \item
    For $K=2$, the p-value is asymptotically equivalent to $c/m$,
    where $c\in[1.37,1.38]$.
  \item
    For $K=3$, the p-value is asymptotically equivalent to $c/m$,
    where $c\in[1.94,1.95]$.
  \end{itemize}
\end{proposition}

The upper index 1 in our notation $\IRP^1$ used for the p-values output by binary IRPs
(as in \eqref{eq:main})
refers to $L:=\left|\mathbf{S}\right|-1$,
the number of boundaries between adjacent summaries (0 and 1 in the binary case).

In the context of Example~\ref{ex:decision},
let us consider a decision tree that outputs 1 (signifying lack of conformity) only rarely,
so that we can expect that $K=0$.
In this case the prediction set output
by the IRP based on this decision tree
will be more confident than the identical prediction set
output by the ICP based on the same inductive nonconformity measure:
the unconfidence of the former will be approximately $0.37/m$ for large $m$,
whereas the unconfidence of the latter will be approximately $1/m$
(the precise value being $1/(m+1)$).

Even if $K=1$, the unconfidence for the IRP
is still close to $0.84/m$,
which is better than the smallest p-value that can be achieved
by any ICP on any training sequence.

In the context of Example~\ref{ex:SVM},
the definition of the nonconformity measure $A$ was chosen
so that $K$ can be expected to be small.
In this case the unconfidence of the IRP
will be significantly better than the unconfidence of the ICP
based on the same inductive nonconformity measure.

\begin{table}
  \caption{The p-values (in $\%$) for binary ICP and IRP for $m=19$ and the first few values of $K$,
    as described in text.}
  \begin{center}
  \begin{tabular}{lcccccccc}
    $K$        & 0     & 1     & 2     & 3     & 4     & 5     & 6     & 7 \\
    \midrule
    ICP        & 5     & 10    & 15    & 20    & 25    & 30    & 35    & 40 \\
    $\IRP^1$   & 1.89  & 4.35  & 7.18  & 10.26 & 13.57 & 17.06 & 20.72 & 24.55 
  \end{tabular}
  \end{center}
  \label{tab:binary}
\end{table}

Table~\ref{tab:binary} gives the unconfidences produced by the ICP and IRP
that are based on the same binary inductive nonconformity measure
and with the IRP based on the nonconformity statistic \eqref{eq:B-binary}.
We take $m:=19$, in order for an ICP to be able to achieve
a statistically significant p-value of $5\%$.
The first two entries in the table,
$1.89\%$ and $4.35\%$,
are less than $5\%$
and so overcome the fundamental limitation of inductive conformal prediction.
If the binary nonconformity measure is intrinsic,
as in Example~\ref{ex:decision} or in the alternative definition in Example~\ref{ex:SVM},
the binary IRP has an obvious significant advantage
over the corresponding ICP, which we may also call binary.
If it extrinsic, comparison is more difficult
since it is more natural to consider the conformal p-values
produced by the unthresholded nonconformity measure.

\begin{table}
  \caption{The asymptotic numerators of the unconfidences for binary ICP and IRP
    for various values of $K$, as described in text.}
  \begin{center}
  \begin{tabular}{lcccccccc}
    $K$      & 0     & 1     & 2     & 3     & 4     & 5     & 6     & 7 \\
    \midrule
    ICP      & 1     & 2     & 3     & 4     & 5     & 6     & 7     & 8 \\
    $\IRP^1  $ & 0.368 & 0.840 & 1.371 & 1.942 & 2.544 & 3.168 & 3.812 & 4.472 \\
    ratio    & 0.368 & 0.420 & 0.457 & 0.486 & 0.509 & 0.528 & 0.545 & 0.559
  \end{tabular}
  \end{center}
  \label{tab:binary-as}
\end{table}

Table~\ref{tab:binary-as} is an asymptotic version of Table~\ref{tab:binary}.
It gives the numerators of asymptotic expressions
such as those in \eqref{eq:K-0} and in the rest of the itemized list
in Proposition~\ref{prop:binary},
with better accuracy and for a wider range of $K$.
The row labelled ``$\IRP^1$'' gives the numerator itself,
and the row labelled ``ratio'' gives the ratio
of the numerator for the IRP to the numerator for the ICP.
Namely,
the asymptotic unconfidence for the hedged prediction set output by the binary IRP is $a_K/m$,
where $a_K$ is given in row ``$\IRP^1$'',
and the asymptotic unconfidence for the ICP is $(K+1)/m$,
with the numerator $K+1$ given in row ``ICP''.
Row ``ratio'' reports $a_K/(K+1)$ showing by how much $a_K/m$ is smaller.
We can see that the ratio is substantially less than 1 even for $K=7$,
in which case we have $4.472/m$ for the IRP (approximately) and $8/m$ for the ICP;
the growth of the ratio quickly slows down as $K$ increases.

\subsection*{A specific extrinsic binary IRP}

As mentioned earlier,
Table~\ref{tab:binary} directly shows the advantage of IRPs over ICPs
for intrinsic binary inductive nonconformity measures.
In practice, however, we are more likely to encounter
a continuous nonconformity measure,
as in the main definition in Example~\ref{ex:SVM}.
We can still make such a nonconformity measure binary
by thresholding, but we might lose something in the process;
in the context of Example~\ref{ex:SVM} we might be better off
directly using the ICP based on the signed distance from the optimal separating hyperplane
divided by the margin
(the sign being $+$ when the object is
on the wrong side of the hyperplane)
as nonconformity measure.
Let's see how much we can lose in a much simpler situation.

\begin{figure}[bt]
  \begin{center}
    \includegraphics[width=0.6\textwidth]{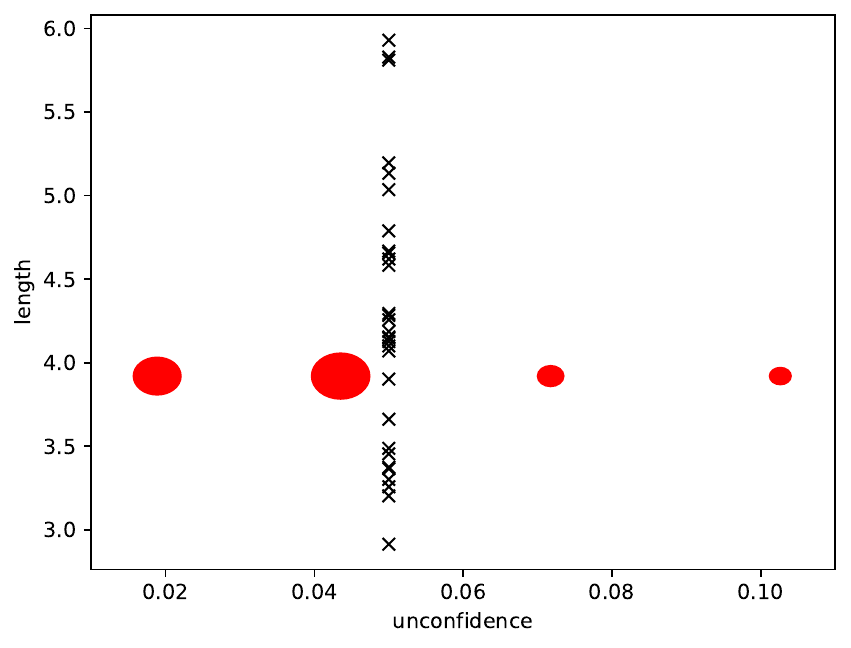}
  \end{center}
  \caption{The binary IRP (red circles) and ICP (black crosses) as described in text}
  \label{fig:binary}
\end{figure}

Figure~\ref{fig:binary} shows some statistics
for predictions output by an ICP and a binary IRP
in an idealized situation with $m=19$.
Suppose that, based on the proper training set,
we can define a ``residual'' $r_i$ for each calibration or test example
that we model as distributed as $r_i\sim N(0,1)$,
and then the corresponding nonconformity score can be defined
as the absolute value of the residual.
For example, the residual can be defined as the difference
$r_i:=y_i-\hat y_i$
between the true and predicted labels
(where the predicted label is based on the proper training set),
and then the nonconformity score is $\alpha_i:=\left|r_i\right|$.
Alternatively, we can define $r_i:=(y_i-\hat y_i)/\hat\sigma_i$,
where $\hat\sigma_i$ is the predicted accuracy of $\hat y_i$,
with the same definition of $\alpha_i$ via $r_i$.
Additionally, we can either transform $(y_i-\hat y_i)/\hat\sigma_i$
monotonically to make its distribution closer to $N(0,1)$
or replace $N(0,1)$ by a different distribution.

To make our inductive nonconformity measure binary,
let us set $U:=\Phi^{-1}(0.975)$
(the $97.5\%$ quantile of the standard Gaussian distribution $N(0,1)$,
whose CDF is denoted by $\Phi$)
and define a new, binary, inductive nonconformity measure
with nonconformity scores $\alpha'_i:=1_{\{\alpha_i\ge U\}}$
(so that $\alpha'_i=1$ with probability $5\%$).
We randomly generate 30 sets of $m$ residuals $r_i\sim N(0,1)$
(representing 30 datasets of size $m=19$)
and compute the corresponding prediction intervals using the ICP and IRP.

Figure~\ref{fig:binary} shows the lengths (see below) of the prediction intervals
output by the ICP as black crosses
and the unconfidences of the prediction intervals output by the binary IRP
as red circles.
The length of a prediction interval is measured ``in the $\alpha$-space'';
e.g., it is literally the length when $r_i:=y_i-\hat y_i$,
and it is measured in units of $\hat\sigma_i$
when $r_i:=(y_i-\hat y_i)/\hat\sigma_i$.
The centre of each black cross has the unconfidence $5\%$
of the corresponding prediction interval as its abscissa,
and the centre of each red circle has the length $2U$
of the corresponding prediction interval as its ordinate.
The area of each red circle is proportional to the number of simulations
(out of 30) that leads to that unconfidence
(found from Table \ref{tab:binary});
therefore, its radius is proportional to the square root of that number
(the precise numbers are:
10 simulations lead to $K=0$,
15 simulations to $K=1$,
3  simulations to $K=2$,
and 2 simulations to $K=3$).
The number of black crosses above the $2U$ level for the IRP is 20 (out of 30).

\begin{figure}[bt]
  \begin{center}
    \includegraphics[width=0.48\textwidth]{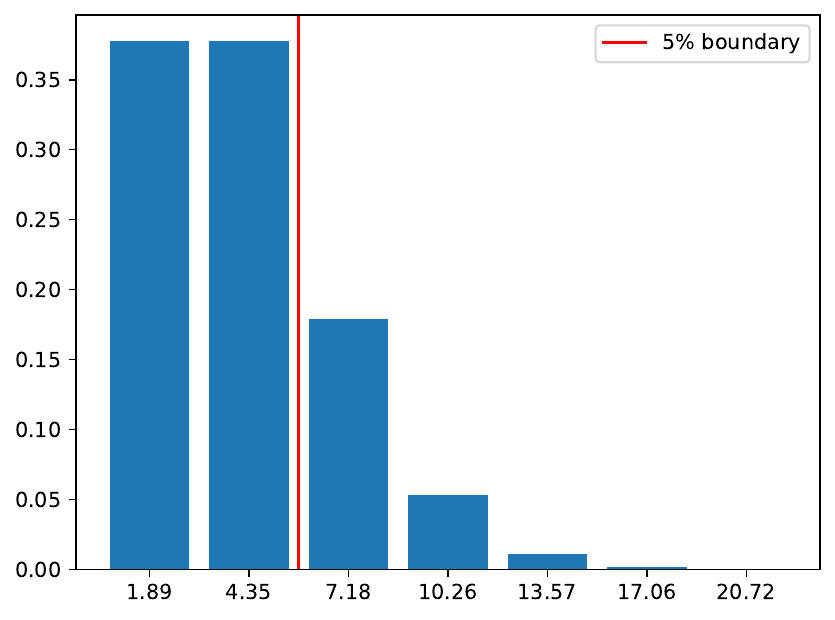}
    \includegraphics[width=0.48\textwidth]{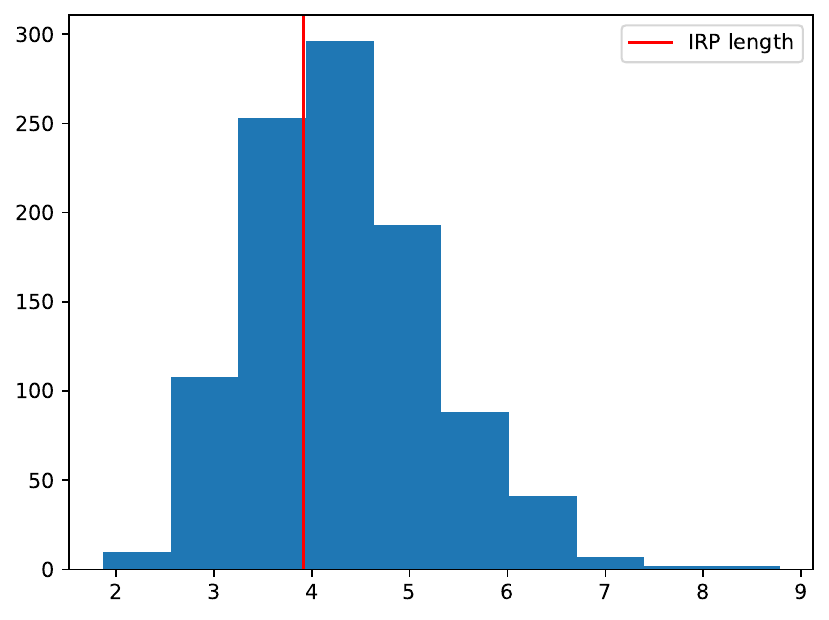}
  \end{center}
  \caption{Details for the same ICP and IRP and the same random data as in Figure~\ref{fig:binary}
    but for $\infty$ (left) or 1000 (right) simulations.
    Left panel: the bar chart for the unconfidences of the prediction intervals output by the IRP.
    Right panel: the histogram of the lengths of the prediction intervals output by the ICP.}
  \label{fig:binary-stats}
\end{figure}

Figure~\ref{fig:binary-stats} gives some statistics
corresponding to Figure~\ref{fig:binary},
but this time with at least 1000 simulations
(for the left panel we even report the exact probabilities,
namely
\[
  \binom{m}{k}
  0.05^k
  0.95^{m-k}
\]
for the probability of $K=k$).
The left (resp.\ right) panel shows that the IRP usually outputs
prediction intervals with better unconfidences (resp.\ lengths),
even for an extrinsic binary inductive nonconformity measure.
The left panel shows that the probability of the IRP
overcoming the fundamental limitation of inductive conformal prediction
is high (it is close to $0.755$).

\section{Inadmissibility of inductive conformal predictors}
\label{sec:inadmissible}

Let us say that an IRP $P_1$ \emph{dominates} an IRP $P_2$
if $P_1\le P_2$
(the p-value output by $P_1$ never exceeds
the p-value output by $P_2$ on the same data).
The domination is \emph{strict} if, in addition,
$P_1(z_1,\dots,z_{n+1})<P_2(z_1,\dots,z_{n+1})$
for some data sequence $z_1,\dots,z_{n+1}$.

An equivalent way to express domination of $P_2$ by $P_1$
is to say that, at each significance level,
the prediction set output by $P_1$ is a subset of
(intuitively, is at least as precise as)
the prediction set output by $P_2$.
Strict domination means that sometimes the prediction set output by $P_1$
is more precise.
An IRP (in particular, an ICP) is \emph{inadmissible}
if it is strictly dominated by another IRP.
This is a special case of the standard notion of inadmissibility
in statistics.

\begin{proposition}\label{prop:inadmissible}
  Any inductive conformal predictor is inadmissible.
\end{proposition}

The idea behind the proof of Proposition~\ref{prop:inadmissible}
given in Sect.~\ref{app:inadmissible}
is that we can improve any ICP $\IRP_{A,\Pi}$
by splitting the summary space $\mathbf{S}$ in two parts
and then boosting $\IRP_{A,\Pi}$ by combining it with the binary IRP
determined by those two parts.
In the next section we will see that ICPs are also inadmissible
in a much stronger sense.

\section{Separation inductive randomness predictors}
\label{sec:SIRP}

In this section we will adapt the idea of binary IRP
to the case of a general summary space $\mathbf{S}$
including that of an interval of the real line
(perhaps infinite in both directions, $\R$,
or in one direction, such as $[c,\infty)$ or $(-\infty,c]$ for some $c\in\R$).
We will consider various ways to split $\mathbf{S}$ in two parts
and thus reduce our prediction problem to a series of binary cases.
Our results will be easiest to interpret if the reader assumes
that $\mathbf{S}$ is an interval of the real line
and the distribution of nonconformity scores $A(z_1,\dots,z_l,Z)$
(where $Z\sim Q$ and $Q^{n+1}$ is the data-generating distribution)
is continuous;
in any case, this is what we will assume
in Proposition~\ref{prop:SIRP} and informal discussions.

A separation IRP is determined by an inductive nonconformity measure $A$
and a 2D array (\emph{threshold array}) $(c_{K,I})$,
$K\in\{0,\dots,m-1\}$ and $I\in\N_1$,
of real numbers in $\mathbf{S}$.
We are particularly interested in the case
where the threshold array $(c_{K,I})$ is dense for each $K\in\{0,\dots,m-1\}$:
for each pair $s_1,s_2\in\mathbf{S}$ such that $s_1<s_2$
there exists $I$ such that $s_1<c_{K,I}<s_2$.
Let us say that $c\in\mathbf{S}$ \emph{$K$-separates the test nonconformity score}
(from the calibration nonconformity scores)
if the test nonconformity score
and exactly $K$ calibration nonconformity scores are above $c$.
The \emph{separation aggregating p-variable} 
based on the threshold array $(c_{K,I})$
is defined to be the aggregating p-variable based on the conformity statistic
\begin{equation}\label{eq:B-sep}
  B(\alpha_{l+1},\dots,\alpha_{n+1})
  :=
  (K,I),
\end{equation}
where $K:=\left|\{i\in\{l+1,\dots,n\}\mid\alpha_i\ge\alpha_{n+1}\}\right|$
and $I$ is the smallest index such that
$c_{K,I}$ $K$-separates the test nonconformity score;
we set $I:=\infty$ if such an index does not exist
(it will usually exist if $K<m$ and the threshold array is dense).
As usual, the order on the possible pairs $(K,I)$ is lexicographic.
The IRP based on the inductive nonconformity measure $A$
and the separation aggregating p-variable based on the threshold array $(c_{K,I})$
will be said to be the \emph{ideal separation IRP based on $A$ and $(c_{K,I})$}.
Remember that the next proposition assumes that the summary space is an interval
and the distribution of nonconformity scores is continuous.

\begin{proposition}\label{prop:SIRP}
  The p-value output by the ideal separation IRP
  based on an inductive nonconformity measure $A$ and a threshold array $(c_{K,I})$
  is at most
  \begin{multline}\label{eq:SIRP}
    \IRP^{\infty}(m,K,I)
    :=
    \frac{K}{m+1}
    +
    \max_{(p_0,\dots,p_I)\in\Delta_I}\\
    \sum_{i=1}^I
    \sum_{k=0}^K
    \frac{m!}{(m-K)!(k+1)!(K-k)!}
    \biggl(
      \sum_{j=0}^{i-1}
      p_j
    \biggr)^{m-K}
    p_i^{k+1}
    \biggl(
      \sum_{j=i+1}^I
      p_j
    \biggr)^{K-k},\\
    K\in\{0,\dots,m-1\},
    \enspace
    I\in\N_1,
  \end{multline}
  where $K$ and $I$ are as defined in \eqref{eq:B-sep}
  and $\Delta_I$ is the standard $I$-simplex
  \[
    \Delta_I
    :=
    \left\{
      (p_0,\dots,p_I)\in[0,\infty)^{I+1}
      \mid
      p_0+\dots+p_I=1
    \right\}.
  \]
  If $I=\infty$,
  the expression in~\eqref{eq:SIRP} is understood to be
  $\IRP^{\infty}(m,K,I):=(K+1)/(m+1)$.
\end{proposition}

See Sect.~\ref{app:SIRP} for a proof of Proposition~\ref{prop:SIRP}.
Let us define a \emph{separation IRP} as predictor
that, under the assumptions of Proposition~\ref{prop:SIRP},
outputs $\IRP^{\infty}(m,K,I)$ as its p-values
for some $A$ and $(c_{K,I})$ (on which it is \emph{based}).
The upper index $\infty$ in $\IRP^{\infty}(m,K,I)$ refers
to the infinite size of the summary space.
In the following two sections,
we will consider IRPs with finite $\mathbf{S}$
based on the separation aggregating p-variable \eqref{eq:B-sep},
which we will also refer to as separation IRPs.

The word ``above'' in the definition of separation IRPs
can be either inclusive or exclusive,
so that ``$\alpha_i$ is above $c_{K,I}$'' may mean either $\alpha_i>c_{K,I}$ or $\alpha_i\ge c_{K,I}$.
(More generally,
we may even replace the 2D array of numbers by a 2D array of closed rays in $\R$
that can be unbounded either on the right or on the left.)
For concreteness, let us use the latter meaning.
Then ``$\alpha_i$ is below $c_{K,I}$'' means $\alpha_i<c_{K,I}$.

The expression $0^0$ in \eqref{eq:SIRP} is treated as $1$.
Therefore, the term in the sum $\sum_{i=1}^I$ corresponding to $i=I$
only contains the term corresponding to $k=K$ in the sum $\sum_{k=0}^K$,
in which the factor $(\dots)^{K-k}$ can be ignored.

\begin{table}
  \caption{Some p-values $\IRP^{\infty}(m,K,I)$ for $m=9$,
    $K=0,1,2$ (corresponding to the conformal p-values of $10\%,20\%,30\%$, respectively),
    and $I=1,\dots,7$.}
  \begin{center}
  \begin{tabular}{lcccccccc}
    $I$   & 1       & 2       & 3       & 4       & 5       & 6       & 7 \\
    \midrule
    $K=0$ & 3.87\%  & 5.53\%  & 6.46\%  & 7.07\%  & 7.49\%  & 7.81\%  & 8.05\%  \\ 
    $K=1$ & 13.02\% & 14.59\% & 15.56\% & 16.23\% & 16.72\% & 17.10\% & 19.74\% \\
    $K=2$ & 22.67\% & 24.17\% & 25.14\% & 25.83\% & 26.34\% & 26.74\% & 29.70\%
  \end{tabular}
  \end{center}
  \label{tab:SIRP}
\end{table}

Table~\ref{tab:SIRP} shows the p-values \eqref{eq:SIRP}
that separation IRPs produce for a calibration sequence of length $m:=9$
when the conformal p-value takes its smallest values $10\%$, $20\%$, or $30\%$.
(We set $m:=19$ only for the simpler binary and, later, ternary IRPs
because results of computations are much less stable for $m=19$
as compared with smaller $m$ such as 9.)
In the case where the conformal p-value takes its smallest value $1/(m+1)$,
$I$ is the smallest index such that all calibration nonconformity scores are below $c_{0,I}$
and the test nonconformity score is above $c_{0,I}$,
and the separation IRP p-value \eqref{eq:SIRP} can then be written as
\begin{equation}\label{eq:SIRP-0}
  \IRP^{\infty}(m,0,I)
  =
  \max_{(p_0,\dots,p_I)\in\Delta_I}
  \sum_{i=1}^I
  \biggl(
    \sum_{j=0}^{i-1}
    p_j
  \biggr)^m
  p_i.
\end{equation}
The smallest possible p-value for separation IRPs corresponds to $I=1$ and is $3.87\%$.
All p-values in the table are below the corresponding conformal p-values,
and later we will see that each separation IRP strictly dominates the corresponding ICP.

To see how separation IRPs could be used for prediction,
we will use a very simple and standard inductive nonconformity measure
\citep[(4.16)]{Vovk/etal:2022book}.

\begin{example}\label{ex:standard}
  \upshape
  Consider the problem of regression, $\mathbf{Y}:=\R$.
  Train a regression model $\hat g:\mathbf{X}\to\R$ (such as a neural network)
  on $z_1,\dots,z_l$ as training sequence.
  Use $A(z_1,\dots,z_l,(x,y)):=\left|y-\hat g(x)\right|$
  as nonconformity measure,
  and set $\mathbf{S}:=[0,\infty)$.
  Let $\alpha_i:=\left|y_i-\hat g(x_i)\right|$,
  $i=l+1,\dots,n$,
  be the $(i-l)$th calibration nonconformity score.
  Arrange these nonconformity scores in the ascending order,
  $\alpha_{(1)}\le\dots\le\alpha_{(m)}$.
  Set $\hat y_{n+1}:=\hat g(x_{n+1})$.
  These are the prediction intervals $\Gamma^{\epsilon}$ (see \eqref{eq:Gamma})
  output by the separation IRP based on a threshold array $(c_{K,I})$ in~$\mathbf{S}$:
  \begin{enumerate}
  \item\label{it:conformal}
    First, we have the conformal prediction intervals:
    \[
      \Gamma^{\frac{K+1}{m+1}}
      =
      \left[
        \hat y_{n+1} - \alpha_{(m-K)},
        \hat y_{n+1} + \alpha_{(m-K)}
      \right].
    \]
  \item\label{it:I-start}
    Let $I_{0,1}$ be the smallest value of $I$ such that $c_{0,I}\in(\alpha_{(m)},\infty)$.
    (Such a value of $I$, here and later in this list,
    will exist under the density requirement
    for the threshold array.)
    Then the longest non-trivial prediction interval is
    \[
      \Gamma^{\IRP^{\infty}(m,0,I_{0,1})}
      =
      \left(
        \hat y_{n+1} - c_{0,I_{0,1}},
        \hat y_{n+1} + c_{0,I_{0,1}}
      \right).
    \]
  \item
    For $j=2,3,\dots$,
    let $I_{0,j}$ be the smallest value of $I$ such that $c_{0,I}\in(\alpha_{(m)},c_{0,I_{0,j-1}})$.
    Then the following prediction intervals are
    \[
      \Gamma^{\IRP^{\infty}(m,0,I_{0,j})}
      =
      \left(
        \hat y_{n+1} - c_{0,I_{0,j}},
        \hat y_{n+1} + c_{0,I_{0,j}}
      \right).
    \]
    This is an inductive definition in $j$.
    If the required value $I$ does not exist
    in any of the items \ref{it:I-start}--\ref{it:I-end},
    the corresponding prediction interval $\Gamma^{\IRP^{\infty}(m,K,I_{K,j})}$
    and any $\Gamma^{\IRP^{\infty}(m,K,I_{K,j'})}$ for $j'>j$
    are undefined at this stage
    (they will be defined in item \ref{it:end} below).
  \item
    For $K=1,\dots,m-1$,
    let $I_{K,1}$ be the smallest value of $I$ such that $c_{K,I}\in(\alpha_{(m-K)},\alpha_{(m-K+1)})$.
    Then
    \[
      \Gamma^{\IRP^{\infty}(m,K,I_{K,1})}
      =
      \left(
        \hat y_{n+1} - c_{K,I_{K,1}},
        \hat y_{n+1} + c_{K,I_{K,1}}
      \right).
    \]
  \item\label{it:I-end}
    Finally, for $K=1,\dots,m-1$ and $j=2,3,\dots$,
    let $I_{K,j}$ be the smallest value of $I$ such that $c_{K,I}\in(\alpha_{(m-K)},c_{K,I_{K,j-1}})$.
    Then the remaining prediction intervals of this type are
    \[
      \Gamma^{\IRP^{\infty}(m,K,I_{K,j})}
      =
      \left(
        \hat y_{n+1} - c_{K,I_{K,j}},
        \hat y_{n+1} + c_{K,I_{K,j}}
      \right).
    \]
  \item\label{it:end}
    For an arbitrary given $\epsilon>0$, define $\Gamma^{\epsilon}$
    as the intersection of all $\Gamma^{\epsilon'}$, $\epsilon'\le\epsilon$,
    defined in the previous items, \ref{it:conformal}--\ref{it:I-end}.
  \end{enumerate}
\end{example}

In the context of Example~\ref{ex:standard},
informal design principles for the threshold array $(c_{K,I})$ are:
we would like $c_{0,I}$ to be situated right above
the typical values of the largest calibration nonconformity score $\alpha_{(m)}$;
we would like $c_{K,I}$, $K=1,2,\dots$, to be situated mostly
inside a typical interval $(\alpha_{(m-K)},\alpha_{(m-K+1)})$
and closer to $\alpha_{(m-K)}$ for small $I$.
To calculate the likely intervals
$(\alpha_{(m)},\infty)$ and $(\alpha_{(m-K)},\alpha_{(m-K+1)})$
we may use the proper training sequence.

To apply a separation IRP predictor,
we need the function $\IRP^{\infty}$ of three variables,
$m$, $K$, and $I$,
defined by \eqref{eq:SIRP}.
Hopefully, for sizeable $m$ the dependence on $m$ will be very predictable;
we find a few asymptotic expressions in the following proposition.
If separation IRPs are ever used in practice,
it makes sense to make the sequence $c_{K,1},c_{K,2},\dots$ finite and short for each $K$.
We can say least about the dependence on $K$.

\begin{proposition}\label{prop:SIRP-values}
  The function $\IRP^{\infty}(m,K,I)$ defined by \eqref{eq:SIRP}
  is increasing in $(K,I)$ (in the sense of the lexicographic order),
  \begin{align}
    &\IRP^{\infty}(m,K,I)
    \in
    \left(
      \frac{K}{m+1},
      \frac{K+1}{m+1}
    \right]
    \text{, and}\label{eq:SIRP-value-1}\\
    &\IRP^{\infty}(m,0,2)
    \sim
    \frac{\exp(\e^{-1}-1)}{m}
    \approx
    \frac{0.531}{m}
    \text{ as $m\to\infty$}.
    \label{eq:SIRP-value-3}
  \end{align}
\end{proposition}

\noindent
The value $\IRP^{\infty}(m,0,1)=\IRP^1(m,0)$ is given by \eqref{eq:K-0}.
Since $\IRP^{\infty}(m,K,I)$ is increasing in $(K,I)$,
it is also increasing in $K$ and $I$ separately.
The approximation $0.531$ in \eqref{eq:SIRP-value-3}
roughly agrees with the value $5.53\%$ given in Table~\ref{tab:SIRP}
(when $m=19$, that value becomes $5.42\%$, and so the agreement becomes better).
See Sect.~\ref{app:SIRP-values} for a proof of Proposition~\ref{prop:SIRP-values}.

Now let us state formally that the separation IRP based on an inductive nonconformity measure $A$
dominates the ICP based on $A$
as corollary of Proposition~\ref{prop:SIRP-values}.
It is then obvious than the domination is usually strict,
which once again demonstrates the inadmissibility of typical ICPs.

\begin{corollary}
  Let $A$ be an inductive nonconformity measure.
  The separation IRP based on $A$ dominates the ICP based on $A$.
\end{corollary}

\begin{proof}
  The statement of the corollary follows from \eqref{eq:SIRP-value-1}.
\end{proof}

However, even separation IRPs are typically inadmissible
and strictly dominated by a calibration-invariant IRP.
Indeed, take any separation IRP
and any sequence $\alpha_{l+1},\dots,\alpha_{n+1}$ of distinct nonconformity scores
such that $\alpha_{n+1}$ is the largest number in this sequence
and $c_{0,1}$ separates it from the calibration nonconformity scores.
The maximum power probability $Q^{m+1}$,
where $Q\in\mathfrak{P}(\mathbf{Z})$,
of the set
\begin{equation}\label{eq:set}
  \left\{
    (\alpha_{\pi(l+1)},\dots,\alpha_{\pi(n)},\alpha_{n+1})
    \mid
    \pi\in\Sym(\{l+1,\dots,n\})
  \right\}
  \subseteq
  \mathbf{S}^{m+1}
\end{equation}
($\Sym(\{l+1,\dots,n\})$ being the family of all permutations of the set $\{l+1,\dots,n\}$)
is
\[
  \frac{m!}{(m+1)^{m+1}}
  \sim
  \sqrt{2\pi/m} \;
  \e^{-m-1},
\]
which is much smaller, for a large $m$,
than the smallest p-value attainable by a separation IRP.
Therefore, we can improve the given separation IRP
by redefining the p-value on the set \eqref{eq:set}.

\begin{remark}
  \upshape
  The non-trivial second addend in the function \eqref{eq:SIRP}
  is defined as the maximum of a homogenous polynomial of degree $m+1$
  over the unit simplex $\Delta_I$.
  This polynomial is not convex in general,
  as can be seen by differentiating the polynomial $p_0^2 p_1$
  that is maximized in $\IRP^{\infty}(2,0,1)$
  (for simplicity, replace $p_1$ by $1-p_0$).
  \iftoggle{FULL}{\bluebegin
    Replacing $p_1$ by $1-p_0$ gives $p_0^2(1-p_0)$,
    whose second derivative $2-6p_0$ changes sign at $p_0=1/3$.
  \blueend}{}%
  Despite the lack of convexity, this is a well-studied problem.
  The problem is NP-complete already for quadratic polynomials,
  but there are PTAS (polynomial-time approximation schemes) for a fixed $m$.
  (See \citealt{deKlerk/etal:2006,deKlerk/etal:2015,deKlerk/etal:2017}.)
\end{remark}

\section{Ternary IRPs}
\label{sec:ternary}

Let us call $\IRP^{\infty}(m,K,\infty)$ the \emph{complete p-value}
output by the separation IRP.
A major weakness of the separation IRPs is that
the complete p-values are just the conformal p-values.
Another manifestation of this weakness
is the limitation mentioned in Sect.~\ref{sec:introduction}:
the best p-value that a separation IRP can achieve is $K/(m+1)$
when the corresponding ICP achieves a p-value of $(K+1)/(m+1)$;
this can be seen directly from~\eqref{eq:SIRP}.
Making the summary space finite removes this limitation.

\begin{figure}[bt]
  \begin{center}
    \includegraphics[width=0.2\textwidth]{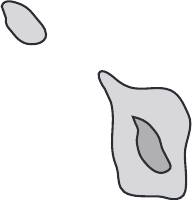}
  \end{center}
  \caption{Example of nested prediction sets
    (casual prediction in dark grey and confident prediction in light grey).}
  \label{fig:nested}
\end{figure}

In this section we discuss \emph{ternary IPRs},
which correspond to $\mathbf{S}$ of size 3.
This case is in the spirit of \citet[Figure~1.5]{Vovk/etal:2022book},
which corresponds to the quaternary case;
a simplified ternary version is shown as Figure~\ref{fig:nested}.
A ternary IRP is defined as separation IRP with $\left|\mathbf{S}\right|=3$,
and we will set $\mathbf{S}:=\{0,1,2\}$
(this particular choice does not restrict generality).
The threshold array $(c_{K,I})$ for it is such that $c_{K,1},c_{K,2}\in\{0.5,1.5\}$
and $c_{K,1}\ne c_{K,2}$ for all $K$;
the values of $c_{K,I}$ for $I>2$ will be irrelevant,
and it will be convenient to regard $I$ to be 1 or 2.
In our experiments we will set
\begin{equation}\label{eq:c}
  c_{K,1}
  :=
  \begin{cases}
    1.5 & \text{if $K<K^*$}\\
    0.5 & \text{otherwise};
  \end{cases}
\end{equation}
intuitively, when $K$ is smaller ($K<K^*$),
we aim for a much smaller p-value and expect $\alpha_{n+1}=2$.

The following proposition gives the p-values
output in the ternary case $\mathbf{S}=\{0,1,2\}$
by the separation aggregating p-variable $P_B$,
where $B$ is defined by \eqref{eq:B-sep}.
It uses the notation
\begin{equation}\label{eq:IRP-2}
  \IRP^2(m,K,I)
  :=
  \URP(B\le(K,I)),
\end{equation}
the possible values for $I$ being 1 and 2.
As before, the upper index in \eqref{eq:IRP-2} is $L:=\left|\mathbf{S}\right|-1$.
We will refer to $\IRP^2(m,K,2)$ as \emph{complete} p-values.

\begin{proposition}\label{prop:ternary}
  The complete p-values of the ternary IRPs are
  \begin{wideformula}\label{eq:2-complete}
    \IRP^2(m,K,2)
    =
    \max_{(p_0,p_1,p_2)\in\Delta_2}
    \NarrowLineBreak
    \sum_{k=0}^K
    \binom{m}{k}
    \left(
      (p_0+p_1)^{m-k}
      p_2^{k+1}
      +
      p_0^{m-k}
      (p_1+p_2)^k
      p_1
    \right).
  \end{wideformula}
  If $c_{K,1}=1.5$,
  \begin{multline}\label{eq:2-case1}
    \IRP^2(m,K,1)
    =
    \max_{(p_0,p_1,p_2)\in\Delta_2}
    \NarrowLineBreak
    \Biggl(
      \sum_{k=0}^{K-1}
      \binom{m}{k}
      \left(
        (p_0+p_1)^{m-k}
        p_2^{k+1}
        +
        p_0^{m-k}
        (p_1+p_2)^k
        p_1
      \right)\\
      +
      \binom{m}{K}
      (p_0+p_1)^{m-K}
      p_2^{K+1}
    \Biggr).
  \end{multline}
  And if $c_{K,1}=0.5$,
  \begin{multline}\label{eq:2-case2}
    \IRP^2(m,K,1)
    =
    \max_{(p_0,p_1,p_2)\in\Delta_2}
    \NarrowLineBreak
    \Biggl(
      \sum_{k=0}^{K-1}
      \binom{m}{k}
      \left(
        (p_0+p_1)^{m-k}
        p_2^{k+1}
        +
        p_0^{m-k}
        (p_1+p_2)^k
        p_1
      \right)\\
      +
      \binom{m}{K}
      p_0^{m-K}
      (p_1+p_2)^K
      p_1
      +
      \binom{m}{K}
      p_0^{m-K}
      p_2^{K+1}
    \Biggr).
  \end{multline}
\end{proposition}

\begin{table}
  \caption{The p-values (in $\%$) for the ternary IRP, $m=19$, and $K=0,\dots,7$, as explained in text.
    The value of $c_{K,1}$ is given as the lower index of $\IRP^2$.}
  \begin{center}
  \begin{tabular}{lcccccccc} 
    $K$                    & 0              & 1              & 2              & 3
                           & 4              & 5              & 6              & 7 \\
    \midrule
    $\IRP^2_{1.5}(19,K,1)$ & \textbf{1.89}  & \textbf{4.73}  & \textbf{8.06}  & \textbf{11.70}
                           & \textbf{15.56} & 19.59          & 23.76          & 28.07 \\
    $\IRP^2_{0.5}(19,K,1)$ & 1.89           & 4.89           & 8.36           & 12.11
                           & 16.07          & \textbf{20.19} & \textbf{24.45} & \textbf{28.84} \\
    $\IRP^2(19,K,2)$       & \textbf{2.71}  & \textbf{6.01}  & \textbf{9.64}  & \textbf{13.49}
                           & \textbf{17.52} & \textbf{21.69} & \textbf{25.99} & \textbf{30.40}
  \end{tabular}
  \end{center}
  \label{tab:ternary}
\end{table}

For a proof, see Sect.~\ref{app:ternary}.
Table~\ref{tab:ternary} gives some numerical values for ternary IRPs and $m:=19$.
(For the corresponding values for the ICP and binary IRP, see Table~\ref{tab:binary}.)
To apply these values,
we need to decide on the values of $c_{K,1}$.
Finding $(K,I)$ as in \eqref{eq:B-sep},
we select the corresponding p-value from:
\begin{itemize}
\item
  the first row of the table if $I=1$ and $c_{K,1}=1.5$;
\item
  the second row of the table if $I=1$ and $c_{K,1}=0.5$;
\item
  the third row of the table if $I=2$.
\end{itemize}
The bold entries in Table~\ref{tab:ternary} correspond
to the threshold array \eqref{eq:c} for $K^*:=5$
(to be used in our experiment later in this section).

As in the binary case,
a ternary IRP based on an intrinsic inductive nonconformity measure
clearly dominates the corresponding ICP;
even the complete p-values (those in the third row of Table~\ref{tab:ternary})
are significantly better than the conformal p-values $0.05(K+1)$.

\begin{remark}
  \upshape
  Ternary IRPs as defined in this section are still inadmissible,
  since we can break extra ties as compared with \eqref{eq:IRP-2}
  if we use $(K,I,J)$ (with lexicographic order)
  as conformity statistic, where
  $J:=\left|\{i\in\{l+1,\dots,n\}\mid\alpha_i=2\}\right|$.
\end{remark}

\subsection*{Specific extrinsic ternary IRP}

\begin{figure}[bt]
  \begin{center}
    \includegraphics[width=0.6\textwidth]{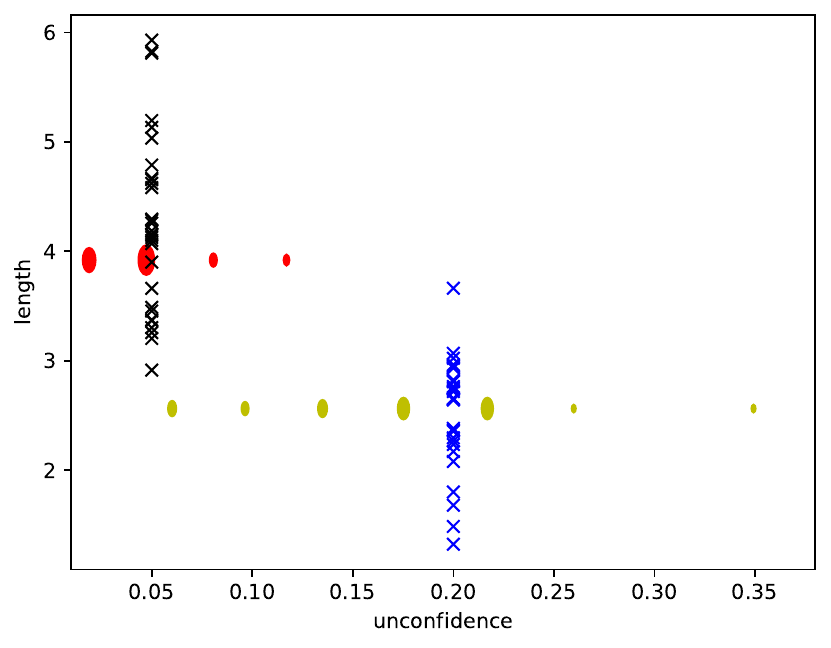}
  \end{center}
  \caption{The ternary IRP (red and yellow circles) and ICP (black and blue crosses) as described in text}
  \label{fig:ternary}
\end{figure}

Now let us discuss in detail a specific ternary IRP
for an extrinsic inductive nonconformity measure,
comparing it with the corresponding ICP.
(Similarly to what we did for binary IRPs in Sect.~\ref{sec:binary}.)
Set $m:=19$ again, and we will again use nonconformity scores $\alpha_i:=\left|r_i\right|$
modelled as $N(0,1)$.
As in Table~\ref{tab:ternary}, we set $K^*:=5$
(although the figure summarizing our results will depend little on the choice of $K^*$).

In the ternary case, we use two thresholds for the nonconformity scores,
which correspond to our convention for casual and confident predictions
in the terminology of \citet[Figure~1.5]{Vovk/etal:2022book}
(confident predictions making an error with probability around 5\%
and casual predictions making an error with probability around 20\%).
The thresholds are chosen in such a way that a random test nonconformity score
is confidently rejected with probability close to 5\%,
is casually (but not confidently) rejected with probability close to 15\%,
and is not rejected at all with probability close to 80\%.
(The validity of the p-values produced by our IRP does not depend
on this informal requirement.)
Namely, we set $U:=\Phi^{-1}(0.975)$ (as in the binary case)
and $U':=\Phi^{-1}(0.9)$.
The ternary inductive nonconformity measure produces nonconformity scores
\begin{equation}\label{eq:alpha'}
  \alpha'_i
  :=
  \begin{cases}
    2 & \text{if $\alpha_i\ge U$}\\
    1 & \text{if $U'\le\alpha_i<U$}\\
    0 & \text{if $\alpha_i<U'$}.
  \end{cases}
\end{equation}
Figure~\ref{fig:ternary} is the analogue of Figure~\ref{fig:binary} for ternary IRPs
with the switch-over $K$ equal to $K^*=5$.
The crosses give the lengths of the prediction intervals produced
by the ICP based on the original inductive nonconformity measure;
the significance level can be read off the horizontal axis
as 5\% (for the black crosses) or 20\% (for the blue crosses).
In both cases the lengths are variable
as they correspond to different calibration sequences.
The red circles correspond to confident prediction intervals
output by the ternary IRP based on \eqref{eq:alpha'},
and the yellow circles correspond to casual prediction intervals
output by those ternary IRP.
The lengths of confident prediction intervals are always $2U$,
since they reject the test labels leading to $\alpha'_{n+1}=2$;
similarly, the lengths of casual prediction intervals are always $2U'$.
What is variable is their unconfidences,
defined as $c$ in \eqref{eq:hedged} and taken from Table~\ref{tab:ternary}.

\begin{figure}[bt]
  \begin{center}
    \includegraphics[width=0.32\textwidth]{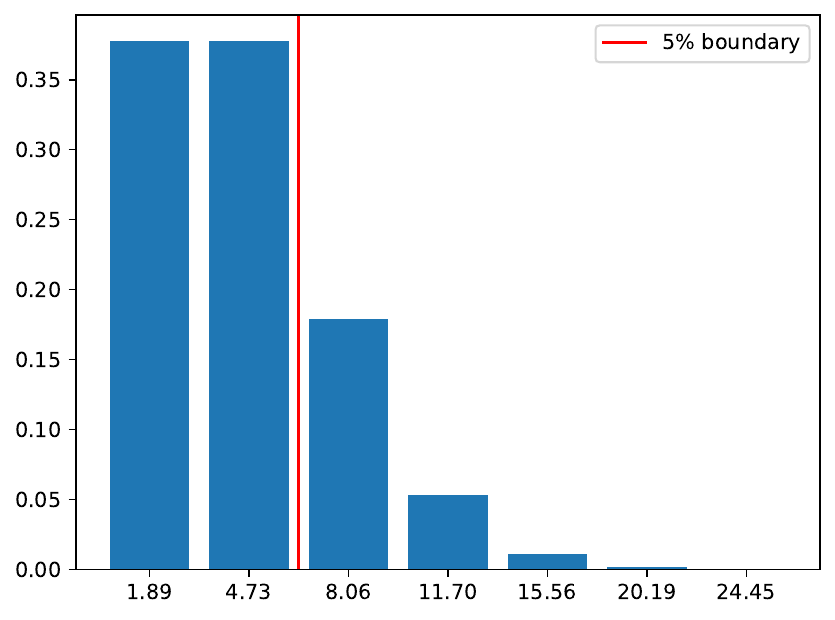}
    \includegraphics[width=0.32\textwidth]{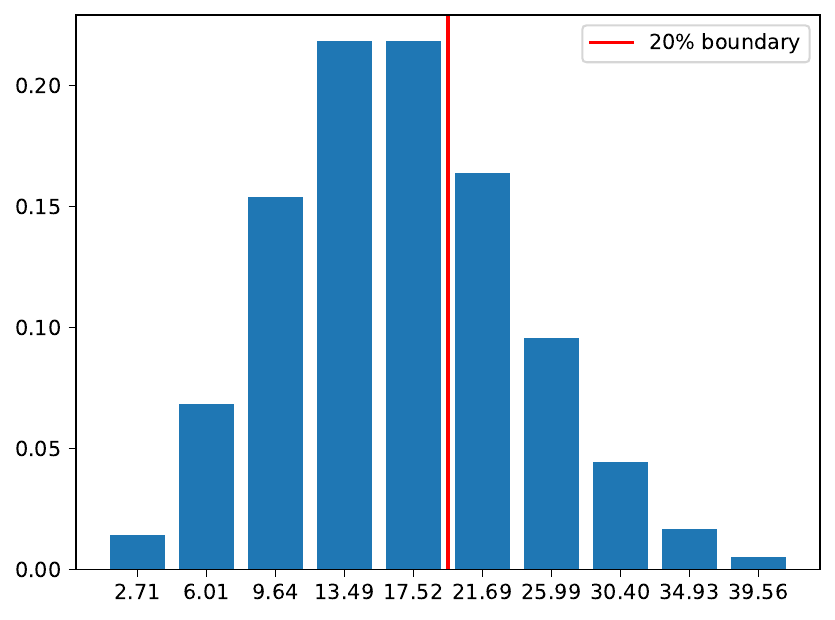}
    \includegraphics[width=0.32\textwidth]{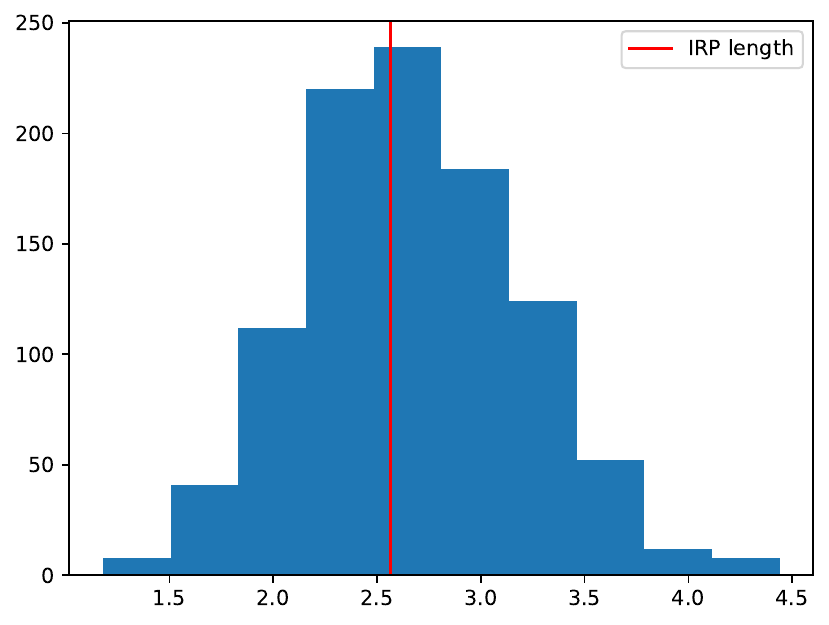}
  \end{center}
  \caption{The analogue of Figure~\ref{fig:binary-stats}
    for the ICP and ternary IRP of Figure~\ref{fig:ternary}.
    Left panel: the bar chart for the unconfidences
    of the confident prediction intervals output by the IRP.
    Middle panel: the analogous bar chart for the casual prediction intervals.
    Right panel: the histogram of the lengths of the prediction intervals
    output by the ICP at $20\%$.}
  \label{fig:ternary-stats}
\end{figure}

Figure~\ref{fig:ternary-stats} is analogous to Figure~\ref{fig:binary-stats},
and its three panels are described in its caption.
The bar chart in the left panel is identical
to the one in the left panel of Figure~\ref{fig:binary-stats}
apart from the labels on the horizontal axis.
There is no need to complement that panel by the histogram for $5\%$,
since we already have it
in Figure~\ref{fig:binary-stats} (right panel).
We can see that the ternary IRP is competitive with the ICP
even for this extrinsic inductive nonconformity measure.

\section{General discrete IRPs}
\label{sec:discrete}

In this section we consider the general case of a finite $\mathbf{S}$
assuming, without loss of generality, $\mathbf{S}=\{0,\dots,L\}$ for $L\in\N_1$.
Then, again without loss of generality, $c_{K,I}$, $I=1,\dots,L$, are all different
and take values in $\{0.5,1.5,\dots,L-0.5\}$.
Define $\IRP^L(m,K,I)$ by \eqref{eq:IRP-2} with $L$ in place of 2.
The \emph{complete} p-values in this context are those, $\IRP^L(m,K,L)$,
corresponding to $I=L$.
The following proposition, to be proved in Sect.~\ref{app:discrete},
only covers the complete p-values;
the other p-values very much depend on the choice of the threshold array.

\begin{proposition}\label{prop:SIRP^L}
  The formula for the complete p-values is
  \begin{equation}\label{eq:SIRP^L}
    \IRP^L(m,K,L)
    =
    \max_{(p_0,\dots,p_L)\in\Delta_L}
    \sum_{k=0}^K
    \binom{m}{k}
    \sum_{J=1}^L
    \left(
      \sum_{j=0}^{J-1}
      p_j
    \right)^{m-k}
    p_J
    \left(
      \sum_{j=J}^L
      p_j
    \right)^k.
  \end{equation}
  For arbitrary but fixed $L$ and $K$ and for $m\to\infty$,
  the optimal p-value \eqref{eq:SIRP^L} is asymptotically equivalent
  to $C/m$, where $C$ is the value of the optimization problem
  \begin{equation}\label{eq:F}
    \sum_{J=1}^L
    c_J
    \sum_{k=0}^K
    \frac{1}{k!}
    \exp
    \left(
      -\sum_{j=J}^L
      c_j
    \right)
    \left(
      \sum_{j=J}^L
      c_j
    \right)^k
    \to\max
  \end{equation}
  whose variables $c_1,\dots,c_L$ range over $[0,\infty)$.
\end{proposition}

Notice that the sum $\sum_{k=0}^K$ in \eqref{eq:F}
is the value at $K$ of the CDF of the Poisson distribution
with parameter $\sum_{j=J}^L c_j$.

In the binary case $L=1$ the optimization problem \eqref{eq:F} becomes
\begin{equation*}
  \sum_{k=0}^K
  \frac{1}{k!}
  \exp(-c)
  c^{k+1}
  \to\max.
\end{equation*}
Setting the derivative in $c$ to 0 and simplifying leads to the equation
\begin{equation}\label{eq:root}
  \sum_{k=0}^K\frac{c^k}{k!}
  =
  \frac{c^{K+1}}{K!}
\end{equation}
for the optimal value of $c$,
and this equation will be used in Sect.~\ref{app:binary} to derive
most of the asymptotic expressions in Proposition~\ref{prop:binary}.

\begin{table}
  \caption{The p-values for $\IRP^L$ in $\%$ for $m=9$, range of $L$, and three values of $K$.}
  \begin{center}
  \begin{tabular}{lcccccccc} 
    $L$     & 1     & 2     & 3     & 4     & 5     & 6     & 7     & 100 \\
    \midrule
    $K=0$   & 3.87  & 5.53  & 6.46  & 7.07  & 7.49  & 7.81  & 8.05  & 9.83 \\
    $K=1$   & 9.05  & 12.37 & 14.13 & 15.22 & 15.96 & 16.51 & 16.92 & 19.74 \\
    $K=2$   & 15.10 & 20.01 & 22.46 & 23.94 & 24.93 & 25.64 & 26.18 & 29.69
  \end{tabular}
  \end{center}
  \label{tab:SIRP-L}
\end{table}

\begin{table}
  \caption{The asymptotic numerators in the p-values for $\IRP^L$ for a range of $L$, and three values of $K$,
  as in Table~\ref{tab:SIRP-L}.}
  \begin{center}
  \begin{tabular}{lcccccccc} 
    $L$     & 1     & 2     & 3     & 4     & 5     & 6     & 7     & 8 \\
    \midrule
    $K=0$   & 0.368 & 0.531 & 0.626 & 0.688 & 0.732 & 0.765 & 0.790 & 0.811 \\
    $K=1$   & 0.840 & 1.171 & 1.352 & 1.467 & 1.547 & 1.606 & 1.651 & 1.686 \\
    $K=2$   & 1.371 & 1.866 & 2.126 & 2.288 & 2.398 & 2.479 & 2.540 & 2.588
  \end{tabular}
  \end{center}
  \label{tab:SIRP-L-as}
\end{table}

Table~\ref{tab:SIRP-L} gives non-asymptotic p-values,
namely for $m:=9$.
The first column of the table (3.87, 9.05, 15.10)
is the analogue of the corresponding entries (1.89, 4.35, 7.18)
in Table~\ref{tab:binary} for a smaller $m$ (9 instead of 19).
Similarly, the second column of the table (5.53, 12.37, 20.01)
is the analogue of the corresponding entries (2.71, 6.01, 9.64)
in Table~\ref{tab:ternary} for a smaller~$m$.

Table~\ref{tab:SIRP-L-as} gives the asymptotic numerators
according to \eqref{eq:F}.
Its first column agrees with the corresponding entries
in Table~\ref{tab:binary-as}.

\section{Conclusion}
\label{sec:conclusion}

In this paper we have defined IRPs and started their study.
Whereas ICPs are inadmissible and are dominated by separation IRPs,
it remains unclear whether separation IRPs,
or other IRPs different from ICPs,
can be useful in practice.

Substantial domination of the separation IRPs constructed in this paper over ICPs
happens only for very small conformal p-values
(first of all, for $\frac{1}{m+1}$).
As the conformal p-value increases
(relative to $\frac{1}{m+1}$, even if it is very small by itself),
separation IRPs quickly become almost indistinguishable from ICPs.
The negative results about randomness predictors
reported in \citet{Vovk:arXiv2502} carry over to IRPs,
but they do not explain this phenomenon.
Can we either strengthen those results or construct better IRPs?

As alluded to in Sect.~\ref{sec:introduction},
this paper is part of a wider research programme,
that of investigating the family of the IRPs that are not ICPs,
in particular establishing its size and usefulness in practice.
The existing results suggesting that
conformal predictors are almost as efficient as randomness predictors
(namely, the results of \citealt{Nouretdinov/etal:2003ALT} and \citealt{Vovk:arXiv2502})
connect conformal predictors and randomness predictors via e-predictors,
which are based on e-values rather than p-values.
This suggests another direction of research:
designing (inductive) randomness e-predictors
that are more efficient than any (inductive) conformal e-predictors.
Besides, efficiency in the sense of producing small p-values
is not the only desideratum in confidence prediction;
it would be interesting to investigate conditionality properties
(in various senses; cf.\ \citealt[Sect.~1.4.4 and Figure~4.8]{Vovk/etal:2022book})
of IRPs.

Finally, we should not forget the advantages of conformal predictors
that are completely lost in non-conformal randomness prediction;
one of them is the independence of errors in the online mode of prediction,
which leads to the possibility of conformal testing.

\subsection*{Acknowledgments}

Many thanks to Alexander Shen for his advice
and to COPA 2025 reviewers for detailed and useful comments.
Computational experiments in this paper used \textsc{Wolfram Mathematica}
and the Python library \texttt{scipy.optimization}.

\appendix
\section{Proofs and complements}
\label{app:proofs}

\subsection{Proof of and complements to Proposition~\ref{prop:binary}}
\label{app:binary}

The following proposition complements the statement of Proposition~\ref{prop:binary}.

\begingroup
\renewcommand{\thetheorem}{\ref{prop:binary}$^\prime$}
\begin{proposition}\label{prop:binary-plus}
  For an arbitrary but fixed $K$ and $m\to\infty$,
  the optimal value of $p$ in \eqref{eq:main} is asymptotically equivalent to $c/m$,
  where $c$ is the unique positive root of the polynomial equation \eqref{eq:root}.
  The p-value \eqref{eq:main} is asymptotically equivalent to
  \begin{equation}\label{eq:numerator}
    \sum_{k=0}^K
    \frac{c^{k+1}\e^{-c}}{k!\,m}
    =
    \frac{c^{K+2}\e^{-c}}{K!\,m}.
  \end{equation}
  A non-asymptotic statement for the difference
  between \eqref{eq:main} and \eqref{eq:numerator} is
  \begin{equation}\label{eq:accuracy}
    \IRP^1(m,K)
    \ge
    \frac{c^{K+2}\e^{-c}}{K!\,m}
    -
    c^2(2\wedge c)
    m^{-2}.
  \end{equation}
  In addition to \eqref{eq:K-0},
  \begin{itemize}
  \item
    for $K=1$, the p-value is asymptotically equivalent (as $m\to\infty$) to
    \begin{equation}\label{eq:K-1}
      \frac{(\phi+\phi^2)\exp(-\phi)}{m}
      =
      \frac{\phi^3\exp(-\phi)}{m}
      \approx
      \frac{0.84}{m},
    \end{equation}
    where $\phi:=(1+\sqrt{5})/2\approx1.62$ is the golden ratio,
  \item
    for $K=2$, the p-value is asymptotically equivalent to
    \begin{equation}\label{eq:K-2}
      \frac{(c^4/2)\exp(-c)}{m}
      \approx
      \frac{1.37}{m},
    \end{equation}
    where
    \[
      c
      :=
      \frac{1+(37-3\sqrt{114})^{1/3}+(37+3\sqrt{114})^{1/3}}{3}
      \approx
      2.27,
    \]
  \item
    and for $K=3$, the p-value is asymptotically equivalent to
    \begin{equation}\label{eq:K-3}
      \frac{(c^5/6)\exp(-c)}{m}
      \approx
      \frac{1.94}{m},
    \end{equation}
    where
    \begin{multline*}
      c
      :=
      \frac14
      +
      \frac14
      \Bigl(
        4(\sqrt{778} - 7)^{1/3} - 36(\sqrt{778} - 7)^{-1/3} + 9
      \Bigr)^{1/2}\\
      +
      \frac12
      \biggl(
        -(\sqrt{778} - 7)^{1/3} + 9(\sqrt{778} - 7)^{-1/3} + \frac92\\
        + 
        \frac{61}
        {
          2 \sqrt{4(\sqrt{778} - 7)^{1/3} - 36(\sqrt{778} - 7)^{-1/3} + 9}
        }
      \biggr)^{1/2}
      \approx
      2.94.
    \end{multline*}
  \end{itemize}
\end{proposition}
\endgroup

Let us now prove Propositions~\ref{prop:binary} and~\ref{prop:binary-plus}.
We can assume, without loss of generality,
that $K<m$ in Proposition~\ref{prop:binary}
(otherwise the statement of the proposition is trivial),
and this assumption then implies
that the inductive nonconformity measure $A$ is a surjection.
Let $B_p$ be the Bernoulli probability measure on $\{0,1\}$
with parameter $p\in[0,1]$:
$B_p(\{1\})=p$.
Since the sequence $\alpha_{l+1},\dots,\alpha_{n+1}$ is IID,
the p-value is the largest probability under $B_p^{m+1}$
of the event of observing at most $K$ 1s among $\alpha_{l+1},\dots,\alpha_n$
and observing $\alpha_{n+1}=1$.
This gives the expression \eqref{eq:main}.

When $K=0$, $\max_p p(1-p)^m$ is attained at $p=\frac{1}{m+1}$,
which leads to \eqref{eq:K-0}.
The inequality
\begin{equation}\label{eq:inequality}
  \frac{m^m}{(m+1)^{m+1}}
  \le
  \frac{\exp(-1)}{m}
\end{equation}
is equivalent to
\begin{equation*}
  \left(
    1 - \frac{1}{m+1}
  \right)^{m+1}
  \le
  \exp(-1)
\end{equation*}
and is easy to check.

When $K=1$,
solving the optimization problem
\begin{equation}\label{eq:objective}
  p(1-p)^m + m p^2(1-p)^{m-1}
  \to
  \max
\end{equation}
leads to a quadratic equation with the solution in $[0,1]$ equal to
\[
  \frac{m-2+\sqrt{5m^2-4m}}{2(m^2-1)}
  \sim
  \frac{\phi}{m}.
\]
Plugging this into the objective function~\eqref{eq:objective}
gives~\eqref{eq:K-1}.

Now let us deal with an arbitrary (but fixed) $K$ and let $m\to\infty$.
The optimal value of $p$ in \eqref{eq:main}
will be of the form $p\sim c/m$ for a constant $c$
(as we will see later in the proof).
Plugging $p\sim c/m$ into the expression following $\max_{p\in[0,1]}$ in \eqref{eq:main},
we can see that this expression is asymptotically equivalent
to the left-hand side of \eqref{eq:numerator}.
Setting the derivative in $c$ of the left hand-side of \eqref{eq:numerator} to 0,
we can check directly,
as we did in Sect.~\ref{sec:discrete},
that the optimal $c$ satisfies the equation~\eqref{eq:root}.
That equation has a unique positive root by Descartes's rule of signs
(see, e.g., \citealt{Wang:2004}).
The uniqueness of a positive root implies
that the left-hand side of \eqref{eq:numerator} attains its maximum at the root.
The equality in \eqref{eq:numerator} follows from~\eqref{eq:root}.
This gives the left-hand sides of \eqref{eq:K-2} and \eqref{eq:K-3}
for $K=2$ and $K=3$.
In these cases, we obtain cubic and quartic equations for $c$, respectively,
and their solutions are given in the statement of Proposition~\ref{prop:binary-plus}.

To obtain the non-asymptotic statement~\eqref{eq:accuracy},
it suffices to set $p:=c/m$ in the expression being maximized
in \eqref{eq:main}
and then to apply Prokhorov's bound $2c(2\wedge c)/m$
\citep[Sect.~3.12]{Shiryaev:2016}
on the total variation distance between the binomial distribution
with parameters $(m,c/m)$
and the Poisson distribution with parameter $c$.

\subsection{Proof of Proposition~\ref{prop:inadmissible}}
\label{app:inadmissible}

Let $A$ be an inductive nonconformity measure;
let us check that we can improve on the corresponding ICP $\IRP_{A,\Pi}$
and define an IRP $\IRP_{A,P}$ strictly dominating $\IRP_{A,\Pi}$.
If $A$ takes only one value, $\IRP_{A,\Pi}$ always outputs 1
and so is clearly inadmissible
(being strictly dominated by the ICP based
on any inductive nonconformity measure taking at least two distinct values).
So let us assume that $A$
takes at least two distinct values,
choose arbitrarily $a\in(\inf A,\sup A)$, and define $P$ as
\begin{multline*}
  P(\alpha_{l+1},\dots,\alpha_{n+1})
  :=\\
  \begin{cases}
    \frac{m^m}{(m+1)^{m+1}} & \text{if $\alpha_{n+1}>a$
      and $\alpha_i<a$ for all $i\in\{l+1,\dots,n\}$}\\
    \Pi(\alpha_{l+1},\dots,\alpha_{n+1}) & \text{otherwise}.
  \end{cases}
\end{multline*}
By inequality~\eqref{eq:inequality},
which also holds with ``$<$'' in place of ``$\le$'',
$P$ can produce p-values that are impossible for ICPs.

It is easy to check that $P$ is an aggregating p-variable:
\begin{itemize}
\item
  when $\epsilon\ge\frac{1}{m+1}$,
  $Q^{m+1}(P\le\epsilon)\le\epsilon$
  follows from $Q^{m+1}(\Pi\le\epsilon)\le\epsilon$
  (since $P$ improves on $\Pi$ only when $\Pi=\frac{1}{m+1}$),
\item
  when $\epsilon<\frac{1}{m+1}$,
  $Q^{m+1}(P\le\epsilon)\le\epsilon$
  follows from the fact that the probability that $B_p^{m+1}$
  produces exactly one 1 and that the 1 is the last bit
  is given by the expression following ``$=$'' in \eqref{eq:K-0}.
\end{itemize}
Therefore, $\IRP_{A,P}$ is an IRP that strictly dominates the ICP $\IRP_{A,\Pi}$.

\subsection{Proof of Proposition~\ref{prop:SIRP}}
\label{app:SIRP}

This proof will show that any separation IRP is a bona fide IRP.
The definition of a separation IRP given in Sect.~\ref{sec:SIRP}
can be restated as follows.
Define $E'_{K,I}$ to be the event that
the conformal p-value is $(K+1)/(m+1)$,
$c_{K,I}$ $K$-separates the test nonconformity score,
and $c_{K,I'}$ does not $K$-separate the test nonconformity score for any $I'<I$.
These events are disjoint,
and setting
\[
  E_{K,I}
  :=
  \cup_{(K',I')\le(K,I)}
  E'_{K',I'},
\]
where ``$\le$'' is the lexicographic order,
we obtain a family of nested events (in the lexicographic order).
The ideal separation IRP then outputs the p-value $\URP(E_{K,I})$
for a given summary space, threshold array, etc.,
and the separation IRP outputs
an upper bound on $\URP(E_{K,I})$.
The notation $E'_{K,I}$ and $E_{K,I}$ will be used repeatedly
in the rest of this section.
We also define $E_{K,\infty}$
to be the event that the conformal p-value is $(K+1)/(m+1)$;
$E_{K,\infty}=\cup_I E_{K,I}$ provided the threshold array is dense.

We are required to check that \eqref{eq:SIRP} is a p-value.
The innermost nested set $E_{0,1}$ is defined as the event
that $c_{0,1}$ separates the test nonconformity score
from the calibration nonconformity scores:
$\alpha_{n+1}\ge c_{0,1}$ while $\alpha_i<c_{0,1}$ for all $i\in\{l+1,\dots,n\}$.
(This corresponds to 0-separation as defined earlier.)
The probability of this event under randomness is $p_0^m p_1$,
where $p_0$ is the probability that $A(z_1,\dots,z_l,Z)\in(-\infty,c_{0,1})$
and $p_1$ is the probability that $A(z_1,\dots,z_l,Z)\in[c_{0,1},\infty)$.
(As before, $Z\sim Q$, where $Q^{n+1}$ is the data-generating distribution.)
This allows us to define
\begin{equation*}
  \IRP^{\infty}(m,0,1)
  =
  \max_{(p_0,p_1)\in\Delta_1}
  p_0^m p_1,
\end{equation*}
in agreement with \eqref{eq:SIRP}.

For a given $I\in\N_1$,
the event $E_{0,I}$ is defined as one of $c_{0,1},\dots,c_{0,I}$
separating the test nonconformity score
from the calibration nonconformity scores.
Let $c_{(1)},\dots,c_{(I)}$ be the sequence $c_{0,1},\dots,c_{0,I}$
sorted in the ascending order;
we extend it by setting $c_{(0)}:=-\infty$ and $c_{(I+1)}:=\infty$.
The probability of the conjunction of the separation
and the test nonconformity score lying in $[c_{(i)},c_{(i+1)})$
is equal to $(p_0+\dots+p_{i-1})^m p_i$,
where $p_j$ is the probability of
$A(z_1,\dots,z_l,Z)\in[c_{(j)},c_{(j+1)})$.
This allows us to set
\begin{wideformula*}
  \IRP^{\infty}(m,0,I)
  =\NarrowLineBreak
  \max_{(p_0,\dots,p_I)\in\Delta_I}
  \left(
    p_0^m p_1
    +
    (p_0+p_1)^m p_2
    +\dots+
    (p_0+\dots+p_{I-1})^m p_I
  \right),
\end{wideformula*}
which again agrees with \eqref{eq:SIRP}
(cf.\ \eqref{eq:SIRP-0}).

Now we assume $K\ge1$.
The event $E_{K,I}$ is defined as the disjunction of the conformal p-value being at most $K/(m+1)$
and the test nonconformity score being $K$-separated from the calibration nonconformity scores
by an element of the set $\{c_{K,1},\dots,c_{K,I}\}$.

We proceed by induction in $K$,
assuming that \eqref{eq:SIRP} works for $K'<K$ in place of $K$.
On the event $E_{K-1,\infty}$ the conformal p-value is at most $K/(m+1)$.
The addend in the first line of \eqref{eq:SIRP} corresponds to (and upper bounds)
the probability of $E_{K-1,\infty}$
under any power probability measure $Q^{m+1}$ on $\mathbf{S}^{m+1}$.
Let us check that the term in the second line of \eqref{eq:SIRP}
corresponds to the probability of the event
$
  E''_{K,I}
  :=
  E_{K,I}
  \setminus
  E_{K-1,\infty}
$
that an element of $c_{K,1},\dots,c_{K,I}$
(or equivalently, of $c_{(1)},\dots,c_{(I)}$,
which are $c_{K,1},\dots,c_{K,I}$ rearranged in the ascending order, as above)
$K$-separates the test nonconformity score.
The index $i$ in the second line of \eqref{eq:SIRP} stands for the part of $E''_{K,I}$
corresponding to $\alpha_{n+1}\in[c_{(i)},c_{(i+1)})$,
and the index $k$ stands for the part of that part
corresponding to there being exactly $k$ calibration nonconformity scores $\alpha_j$,
$j\in\{l+1,\dots,n\}$,
such that $\alpha_j\in[c_{(i)},c_{(i+1)})$ and $\alpha_j\ge\alpha_{n+1}$.
The second line of \eqref{eq:SIRP} is obtained by the multiplication of several terms:
\begin{itemize}
\item
  the probability that exactly $m-K$ calibration nonconformity scores
  with specified indices are below $c_{(i)}$ is
  \[
    \biggl(
      \sum_{j=0}^{i-1}
      p_j
    \biggr)^{m-K};
  \]
\item
  the probability that exactly $K-k$ of the remaining $K$ calibration nonconformity scores
  with specified indices are above $c_{(i+1)}$ is
  \[
    \biggl(
      \sum_{j=i+1}^I
      p_j
    \biggr)^{K-k};
  \]
\item
  the probability that the remaining $k$ calibration nonconformity scores
  and the test nonconformity score are in $[c_{(i)},c_{(i+1)})$
  is $p_i^{k+1}$;
\item
  the conditional probability (given the event in the previous item)
  that all those $k$ calibration nonconformity scores
  are above the test nonconformity score is $1/(k+1)$
  (this assumes that the calibration and test nonconformity scores are all different
  and is the only place in this proof
  where we use the assumption of continuity of the inductive nonconformity scores);
\item
  finally, there are
  \[
    \binom{m}{m-K}
    \binom{K}{K-k}
  \]
  ways to specify the positions of the indices in the first two items.
\end{itemize}

Let us check that the convention in the statement of the proposition about the case $I=\infty$
agrees with \eqref{eq:SIRP} provided the threshold array is dense
(or at least non-trivial in a weak sense).
It is easier to use the derivation of \eqref{eq:SIRP}
than \eqref{eq:SIRP} itself.
Letting $I\to\infty$, we can take all $p_j=1/(I+1)$ equal and shrinking to 0,
and then the probability of $E''_{K,I}$ will tend to $1/(m+1)$
(the probability that the rank of the last observation is $K+1$
in an IID series of $m+1$ continuously distributed observations).
On the other hand, the upper bound of $(K+1)/(m+1)$ on $\IRP^{\infty}(m,K,\infty)$
is obvious.

\subsection{Proof of Proposition~\ref{prop:SIRP-values}}
\label{app:SIRP-values}

The monotonicity of $\IRP^{\infty}$ immediately follows from its definition.

Let us check \eqref{eq:SIRP-value-3}.
For $K=0$ and $I=2$ our optimization problem (see \eqref{eq:SIRP-0}) can be written as
\begin{equation}\label{eq:objective-2}
  p_0^m (1-p_0-p_2) + (1-p_2)^m p_2
  \to
  \max
\end{equation}
(after substituting $1-p_0-p_2$ for $p_1$).
Setting the partial derivatives of the objective function in $p_0$ and $p_2$ to 0
we obtain
\[
  p_0
  =
  1 + \frac{\e^{-1}-2}{m} + O(m^{-2}),
  \quad
  p_2
  =
  \frac{1-\e^{-1}}{m} + O(m^{-2})
\]
(so that $p_0+p_2<1$ asymptotically, as it should).
Plugging this into the objective function in \eqref{eq:objective-2} gives
\begin{multline*}
  \IRP^{\infty}(m,0,2)
  =
  \left(
    1 + \frac{\e^{-1}-2}{m} + O(m^{-2})
  \right)^m
  \left(
    \frac1m + O(m^{-2})
  \right)\\
  +
  \left(
    1 + \frac{\e^{-1}-1}{m} + O(m^{-2})
  \right)^m
  \left(
    \frac{1-\e^{-1}}{m} + O(m^{-2})
  \right)\\
  =
  \frac{\exp(\e^{-1}-2)}{m}
  +
  \frac{\exp(\e^{-1}-1)(1-\e^{-1})}{m}
  +
  O(m^{-2}) 
  =
  \frac{\exp(\e^{-1}-1)}{m}
  +
  O(m^{-2}).
  \iftoggle{CONF}{}{\qedhere}
\end{multline*}

\subsection{Proof of Proposition~\ref{prop:ternary}}
\label{app:ternary}

There is no need to prove \eqref{eq:2-complete},
since this is a special case of Proposition~\ref{prop:SIRP^L},
which we prove in Sect.~\ref{app:discrete}.

The addends of the sum over $k$
in the first lines of \eqref{eq:2-case1} and \eqref{eq:2-case2}
give the probability of the event
that the conformal p-value is $K/(m+1)$ or less
if the distribution of the nonconformity score
$A(z_1,\dots,z_l,Z)$ is $(p_0,p_1,p_2)$
(meaning that the probability of it taking value $j\in\{0,1,2\}$ is $p_j$).
The addend in the second line of \eqref{eq:2-case1}
is the probability that $1.5$ $K$-separates the test nonconformity score.
The two addends in the second line of \eqref{eq:2-case1}
are the probabilities of the two disjoint possibilities
for $0.5$ $K$-separating the test nonconformity score
while the conformal p-value is $(K+1)/(m+1)$:
the first possibility is where $\alpha_{n+1}=1$,
and the second (typically unlikely) possibility
is where $\alpha_{n+1}=2$
(in which case there are no $\alpha_i=1$ among $\alpha_{l+1},\dots,\alpha_{n+1}$).

\subsection{Proof sketch of Proposition~\ref{prop:SIRP^L}}
\label{app:discrete}

Let us first check \eqref{eq:SIRP^L}.
Each addend in the sum over $k$ in \eqref{eq:SIRP^L}
is the probability of the event that the conformal p-value is $(k+1)/(m+1)$
(with the same interpretation of $p_0,p_1,\dots$ as in previous sections).
For a given $k$,
each addend in the sum over $J$ in \eqref{eq:SIRP^L}
is the probability that the conformal p-value is $(k+1)/(m+1)$,
the test nonconformity score is $\alpha_{n+1}=J$,
$m-k$ calibration nonconformity scores at specified positions are less than $J$,
and the remaining $k$ calibration nonconformity scores are $J$ or more.

The proof of \eqref{eq:F} will use the following complementary statement.

\begingroup
\renewcommand{\thetheorem}{\ref{prop:SIRP^L}$^\prime$}
\begin{proposition}\label{prop:SIRP^L-plus}
  For arbitrary but fixed $L$ and $K$ and for $m\to\infty$,
  an optimal vector $p$ in \eqref{eq:SIRP^L} satisfies
  \[
    p_j\sim\frac{c_j}{m},
    \enspace
    j=1,\dots,L,
    \quad
    p_0
    =
    1
    -
    \frac{\sum_{j=1}^L c_j}{m}
    +
    o(1/m),
  \]
  where an optimal $c=(c_j)_{j=1}^L\in[0,\infty)^L$ delivers a solution
  to the optimization problem \eqref{eq:F}.
\end{proposition}
\endgroup

To establish \eqref{eq:F} and Proposition~\ref{prop:SIRP^L-plus},
first of all notice that the $\max$ in \eqref{eq:F} is attained:
if any of $c_j$ tends to infinity,
the objective function will tend to $0$,
so that we are effectively maximizing over a compact set.
Now define new variables $c_j$, $j=1,\dots,L$, by $p_j=c_j/m$
as suggested by Proposition~\ref{prop:SIRP^L-plus}.
Plugging this into the expression following the $\max$ in \eqref{eq:SIRP^L}
and assuming $p_0=1-\sum_{j=1}^L p_j\to1$,
we obtain
\begin{multline}\label{eq:chain}
  \sum_{k=0}^K
  \binom{m}{k}
  \sum_{J=1}^L
  \left(
    \sum_{j=0}^{J-1}
    p_j
  \right)^{m-k}
  p_J
  \left(
    \sum_{j=J}^L
    p_j
  \right)^k
  \NarrowLineBreak
  \sim
  \sum_{J=1}^L
  p_J
  \sum_{k=0}^K
  \frac{m^k}{k!}
  \left(
    1
    -
    \sum_{j=J}^L
    p_j
  \right)^m
  \left(
    \sum_{j=J}^L
    p_j
  \right)^k\\
  \le
  \sum_{J=1}^L
  \frac{c_J}{m}
  \sum_{k=0}^K
  \frac{1}{k!}
  \exp
  \left(
    -\sum_{j=J}^L c_j
  \right)
  \left(
    \sum_{j=J}^L
    c_j
  \right)^k.
\end{multline}
The asymptotic equivalence ``$\sim$'' holds uniformly over all $p$
for which the expression to its left is $m^{-2}$ or more:
indeed, the addends in the sum $\sum_{k,J}$ to the left of ``$\sim$'' for which
\[
  \sum_{j=0}^{J-1}
  p_j
  \le
  1-m^{-1/2}
\]
are negligible since
\[
  \left(
    1 - m^{-1/2}
  \right)^m
  \le
  \exp
  \left(
    -m^{1/2}
  \right)
\]
shrinks to 0 super-polynomially fast as $m\to\infty$,
and
\[
  \left(
    \sum_{j=0}^{J-1}
    p_j
  \right)^{m-k}
  \sim
  \left(
    \sum_{j=0}^{J-1}
    p_j
  \right)^m
  =
  \left(
    1
    -
    \sum_{j=J}^{L}
    p_j
  \right)^m
\]
for the other addends.
Because of the inequality ``$\le$'' in the chain \eqref{eq:chain},
we can assume that $(c_j)$ range over a compact set,
and then the ``$\le$'' can be replaced by ``$\sim$''.
It remains to compare the last expression with \eqref{eq:SIRP^L}.
\end{document}